\documentclass{article}
\usepackage{preprint_arxiv}

\newcommand{\dotprod}[2]{#1 \cdot #2}
\renewcommand{\exp}[1]{e^{#1}}
\newcommand{\lexp}[1]{{\rm exp}\left({#1}\right)}

\newcommand{\vzero}{\vectsym{0}}
\newcommand{\eps}{\varepsilon}
\renewcommand{\O}{\mathcal{O}}
\newcommand{\X}{\mathcal{X}}
\newcommand{\Y}{\mathcal{Y}}
\renewcommand{\H}{\mathcal{H}}

\newcommand{\W}{\mathcal{W}}
\newcommand{\PS}{\mathcal{P}}
\newcommand{\yp}{y'}
\newcommand{\yh}{\hat{y}}
\renewcommand{\wp}{w'}
\newcommand{\m}{n'}
\newcommand{\T}{\mathbb{T}}
\newcommand{\Rademacher}{\mathfrak{R}}
\newcommand{\G}{\mathfrak{G}}
\newcommand{\hh}{\hat{h}}

\newcommand{\hp}{h'}

\newcommand{\Rp}{R'}
\newcommand{\rp}{r'}

\newcommand{\A}{\mathcal{A}}
\newcommand{\Sp}{S'}
\newcommand{\pip}{\pi'}
\newcommand{\hns}{\hspace{-0.025in}}
\newcommand{\s}{\mathfrak{s}}
\newcommand{\ft}{\widetilde{f}}
\newcommand{\Fcal}{\mathcal{F}}
\newcommand{\Gcal}{\mathcal{G}}
\newcommand{\hs}{h^*}

\newcommand{\E}{\mathbb{E}}
\let\P\undefined
\newcommand{\P}{\mathbb{P}}

\newcommand{\R}{\mathbb{R}}

\newcommand{\BigO}[1]{\ensuremath{\mathcal{O}(#1)}}                             % Big Oh.
\newcommand{\vectsym}[1]{\ensuremath{\boldsymbol{#1}}}                          % vector for symbols.
\newcommand{\KL}[2]{\ensuremath{\mathbb{KL} \left(#1 \middle\Vert #2 \right)}}   % KL divergence.
\newcommand{\Exp}[2]{\ensuremath{\E_{#1}\left[#2\right]}}                % Expectation.
\newcommand{\Ind}[1]{\ensuremath{\mathbbm{1} \hspace{-0.03in} \left[#1\right]}}                     % Indicator function.
\newcommand{\Inn}[1]{\ensuremath{\langle #1 \rangle}}

\newcommand{\NormII}[1]{\ensuremath{\lVert #1 \rVert}_2}              % L2 Norm 
\newcommand{\ANormII}[1]{\ensuremath{\left\lVert #1 \right\rVert}_2}              % L2 Norm 
\newcommand{\InNorm}[1]{{\left\vert\kern-0.2ex\left\vert\kern-0.2ex\left\vert #1 
    \right\vert\kern-0.2ex\right\vert\kern-0.2ex\right\vert}}                    % Induced Norm.

\newcommand{\InNormII}[1]{{\left\vert\kern-0.2ex\left\vert\kern-0.2ex\left\vert #1 
    \right\vert\kern-0.2ex\right\vert\kern-0.2ex\right\vert}_2}                    % Induced 2 Norm (Spectral Norm).

\newcommand{\InNormInfty}[1]{{\left\vert\kern-0.2ex\left\vert\kern-0.2ex\left\vert #1 
    \right\vert\kern-0.2ex\right\vert\kern-0.2ex\right\vert}_{\infty}}           % Induced Infinity norm.

\newcommand{\iid}{i.i.d.~}                                                        % IID.
\DeclarePairedDelimiterX{\Inner}[2]{\langle}{\rangle}{#1, #2}                    % Inner product

\newcommand{\ceil}[1]{\left\lceil#1\right\rceil}

\DeclareMathOperator*{\argmax}{argmax}

\newtheorem{assumption}{Assumption}

\newtheorem{lemma}{Lemma}
\newtheorem{claim}{Claim}

\newtheorem{theorem}{Theorem}
\newtheorem{remark}{Remark}

\newcommand{\vw}{\vectsym{w}}

\pdfoutput=1
\allowdisplaybreaks

\author{
  Kevin Bello \\
  Department of Computer Science\\
  Purdue University\\
  West Lafayette, IN, USA \\
  \texttt{kbellome@purdue.edu} \\
   \And
  Jean Honorio \\
  Department of Computer Science\\
  Purdue University\\
  West Lafayette, IN, USA \\
   \texttt{jhonorio@purdue.edu} \\
}

\title{Learning latent variable structured prediction models with Gaussian perturbations}

\begin{document}

\maketitle

%%%%%%%%%%%%%%%%%%%%%%%%%%%%%%%%%%%%%%%%%%%%%%%%%%%%%%%%%%%%%%%%%%%%%%%%%%%%%%%%%%%%%%%%%%%%%%%%%%%%%%%%%
%%%%%%%%%%%%%%%	ABSTRACT
%%%%%%%%%%%%%%%%%%%%%%%%%%%%%%%%%%%%%%%%%%%%%%%%%%%%%%%%%%%%%%%%%%%%%%%%%%%%%%%%%%%%%%%%%%%%%%%%%%%%%%%%%

\begin{abstract}
	The standard margin-based structured prediction commonly uses a maximum loss over \textit{all} possible structured outputs \cite{tsochantaridis2005large,Altun03,Collins04b,Taskar03}. 
	The large-margin formulation including latent variables  \cite{Yu09,Ping14} not only results in a non-convex formulation but also increases the search space by a factor of the size of the latent space.
	Recent work \cite{honorio2016} has proposed the use of the maximum loss over \textit{random} structured outputs sampled independently from some proposal distribution, with theoretical guarantees.
	We extend this work by including latent variables. We study a new family of loss functions under Gaussian perturbations and analyze the effect of the latent space on the generalization bounds.
	We show that the non-convexity of learning with latent variables originates naturally, as it relates to a tight upper bound of the Gibbs decoder distortion with respect to the latent space.
	Finally, we provide a formulation using random samples that produces a tighter upper bound of the Gibbs decoder distortion up to a statistical accuracy, which enables a faster evaluation of the objective function. 
	We illustrate the method with synthetic experiments and a computer vision application.
\end{abstract}

%%%%%%%%%%%%%%%%%%%%%%%%%%%%%%%%%%%%%%%%%%%%%%%%%%%%%%%%%%%%%%%%%%%%%%%%%%%%%%%%%%%%%%%%%%%%%%%%%%%%%%%%%
%%%%%%%%%%%%%%%	INTRO
%%%%%%%%%%%%%%%%%%%%%%%%%%%%%%%%%%%%%%%%%%%%%%%%%%%%%%%%%%%%%%%%%%%%%%%%%%%%%%%%%%%%%%%%%%%%%%%%%%%%%%%%%

\section{Introduction}
	Structured prediction is of high interest in many domains such as computer vision \cite{nowozin2011structured}, natural language processing \cite{Zhang14,Zhang15}, and computational biology \cite{li2007protein}. 
	Some standard methods for structured prediction are conditional random fields (CRFs) \cite{lafferty2001conditional} and structured SVMs (SSVMs) \cite{Taskar03,tsochantaridis2005large}.
	
	In many tasks it is crucial to take into account latent variables.
	For example, in machine translation, one is usually given a sentence $x$ and its translation $y$, but not the linguistic structure $h$ that connects them (e.g. alignments between words).
	Even if $h$ is not observable it is important to include this information in the model in order to obtain better prediction results.
	Examples also arise in computer vision, for instance, most images in indoor scene understanding \cite{wang2013discriminative} are cluttered by furniture and decorations, whose appearances vary drastically across scenes, and can hardly be modeled (or even hand-labeled) consistently. In this application, the input $x$ is an image, the structured output $y$ is the layout of the faces (floor, ceiling, walls) and furniture, while the latent structure $h$ assigns a binary label to each pixel (clutter or non-clutter.)
	
	During past years, there has been several solutions to address the problem of latent variables in structured prediction. 
	In the field of computer vision, hidden conditional random fields (HCRF) \cite{quattoni2007hidden,wang2006hidden,quattoni2005conditional} have been widely applied for object recognition and gesture detection. 
	In natural language processing there is also work in applying discriminative probabilistic latent variable models, for example the training of probabilistic context free grammars with latent annotations in a discriminative manner \cite{petrov2008discriminative}. 
	The work of \citet{Yu09} extends the margin re-scaling SSVM in \cite{tsochantaridis2005large} by introducing latent variables (LSSVM) and obtains a formulation that is optimized using the concave-convex procedure (CCCP) \cite{yuille2002concave}. 
	The work of \citet{Ping14} considers a smooth objective in LSSVM by incorporating marginal maximum \textit{a posteriori} inference that ``averages'' over the latent space.
	
	Some of the few works in deriving generalization bounds for structured prediction include the work of \citet{McAllester07}, which provides PAC-Bayesian guarantees for arbitrary losses, and the work of \citet{cortes2016structured}, which provides data-dependent margin guarantees for a general family of hypotheses, with an arbitrary factor graph decomposition. 
	However, with the exception of \cite{honorio2016}, both aforementioned works do not focus on producing computationally appealing methods. Moreover, prior generalization bounds have not focused on latent variables.

	\paragraph{Contributions.}
	We focus on the learning aspects of structured prediction problems using latent variables. 
	We first extend the work of \cite{McAllester07} by including latent variables, and show that the non-convex formulation using the slack re-scaling approach with latent variables is related to a tight upper bound of the \textit{Gibbs decoder distortion}. 
	This motivates the apparent need of the non-convexity in different formulations using latent variables (e.g., \cite{Yu09, hinton2012practical}).
	Second, we provide a tighter upper bound of the Gibbs decoder distortion by randomizing the search space of the optimization problem. 
	That is, instead of having a formulation over all possible structures and latent variables (usually exponential in size), we propose a formulation that uses \iid samples coming from some proposal distribution. 
	This approach is also computationally appealing in cases where the latent space is polynomial in size, since it would lead to a fully polynomial time evaluation of the formulation.
	The use of standard Rademacher arguments and the analysis of \cite{honorio2016} would lead to a prohibitive upper bound that is proportional to the size of the latent space.
	We provide a way to obtain an upper bound that is logarithmic in the size of the latent space.
	Finally, we provide experimental results in synthetic data and in a computer vision application, where we obtain improvements in the average test error with respect to the values reported in \cite{gane2014learning}.

%%%%%%%%%%%%%%%%%%%%%%%%%%%%%%%%%%%%%%%%%%%%%%%%%%%%%%%%%%%%%%%%%%%%%%%%%%%%%%%%%%%%%%%%%%%%%%%%%%%%%%%%%
%%%%%%%%%%%%%%%	PRELIMINARIES
%%%%%%%%%%%%%%%%%%%%%%%%%%%%%%%%%%%%%%%%%%%%%%%%%%%%%%%%%%%%%%%%%%%%%%%%%%%%%%%%%%%%%%%%%%%%%%%%%%%%%%%%%

\section{Background}
\label{sec:preliminaries}

	We denote the input space as $\X$, the output space as $\Y$, and the latent space as $\H$. We assume a distribution $D$ over the observable space $\X \times \Y$. We further assume that we are given a training set $S$ of $n$ \iid samples drawn from the distribution $D$, i.e., $S \sim D^n$.  
	
	Let $\Y_x \neq \emptyset$ denote the countable set of feasible outputs or \textit{decodings} of $x$. In general, $|\Y_x|$ is exponential with respect to the input size. Likewise, let $\H_x \neq \emptyset$ denote the countable set of feasible latent decodings of $x$. 
	
	We consider a fixed mapping $\Phi$ from triples to feature vectors to describe the relation among input $x$, output $y$, and latent variable $h$, i.e., for any triple $(x,y,h)$, we have the feature vector $\Phi(x,y,h) \in \R^k \setminus \{0\}$. For a parameter $\vw \in \W \subseteq \R^k \setminus \{0\}$, we consider linear decoders of the form:
	\begin{align} \label{eq:inferenceall}
		f_{\vw} (x) = \argmax_{(y,h) \in \Y_x \times\H_x} \dotprod{ \Phi(x,y,h) }{ \vw }.
	\end{align}
	The problem of computing this $\argmax$ is typically referred as the \textit{inference} or \textit{prediction} problem.
	In practice, very few cases of the above general inference problem are tractable, while most are NP-hard and also hard to approximate within a fixed factor.
	(For instance, see Section 6.1 in \cite{honorio2016} for a thorough discussion.)
	
	We denote by $d: \Y \times \Y \times \H \to [0,1]$ the \textit{distortion} function, which measures the dissimilarity among two elements of the output space $\Y$ and one element of the latent space $\H$. 
	(Note that the distortion function is general in the sense that the latent element may not be used in some applications.)
	Therefore, the goal is to find a $\vw \in \W$ that minimizes the decoder distortion, that is:
	\begin{align} \label{eq:nonrobust_distortion}
		\min_{\vw \in \W} \Exp{(x,y) \sim D}{d(y,\Inn{f_{\vw}(x)})}.
	\end{align} 
	In the above equation, the angle brackets indicate that we are inserting a pair $(\yh, \hh) = f_{\vw}(x)$ into the distortion function.
	From the computational point of view, the above optimization problem is intractable since $d(y,\Inn{f_{\vw}(x)})$ is discontinuous with respect to $\vw$. From the statistical viewpoint, eq.\eqref{eq:nonrobust_distortion} requires access to the data distribution $D$ and would require an infinite amount of data. In practice, one only has access to a finite number of samples.
	
	Furthermore, even if one were able to compute $\vw$ using the objective in eq.\eqref{eq:nonrobust_distortion}, this parameter $\vw$, while achieving low distortion, could potentially be in a neighborhood of parameters with high distortion.
	Therefore, we can optimize a more \textit{robust} objective that takes into account perturbations.
	In this paper we consider Gaussian perturbations. 
	More formally, let $\alpha > 0$ and let $Q(\vw)$ be a unit-variance Gaussian distribution centered at $\alpha \vw$ of parameters $\vw' \in \W$. The Gibbs decoder distortion of the perturbation distribution $Q(\vw)$ and data distribution $D$, is defined as:
	\begin{align} \label{eq:gibbs_distortion}
		L(Q(\vw),D) = \Exp{(x,y) \sim D}{\Exp{\vw' \sim Q(\vw)}{d(y,\Inn{f_{\vw'}(x)}) } }
	\end{align}
	Then, the optimization problem using the Gibbs decoder distortion can be written as:
	\[
	\min_{\vw \in W} L(Q(\vw),D).
	\]
	
	We define the margin $m(x,y,\yp,\hp,\vw)$ as follows:
	\[
	m(x,y,\yp,\hp,\vw) = \max_{h \in \H_x} \dotprod{\Phi(x,y,h)}{\vw} - \dotprod{\Phi(x,\yp,\hp)}{\vw}.
	\]
	Note that since we are considering latent variables, our definition of margin differs from \cite{McAllester07,honorio2016}.
	Let $\hs = \argmax_{h \in \H_x} \dotprod{\Phi(x,y,h)}{\vw}$. 
	In this case $\hs$ can be interpreted as the latent variable that best explains the pair $(x,y)$. 
	Then, for a fixed $\vw$, the margin computes the amount by which the pair $(y,\hs)$ is preferred to the pair $(\yp,\hp)$.

	Next we introduce the concept of ``parts'', also used in \cite{McAllester07}. 
	Let ${c(p,x,y,h)}$ be a nonnegative integer that gives the number of times that the part ${p \in \PS}$ appears in the triple ${(x,y,h)}$.
	For a part ${p \in \PS}$, we define the feature $p$ as follows:
	\begin{align*}
		\Phi_p(x,y,h) \equiv c(p,x,y,h)
	\end{align*}
	We let ${\PS_x \neq \emptyset}$ denote the set of ${p \in \PS}$ such that there exists ${(y,h) \in \Y_x \times \H_x}$ with ${c(p,x,y,h)>0}$.
	
	\paragraph{Structural SVMs with latent variables.}
	
	\cite{Yu09} extend the formulation of \textit{margin re-scaling} given in \cite{tsochantaridis2005large} incorporating latent variables. 
	The motivation to extend such formulation is that it leads to a difference of two convex functions, which allows the use of CCCP \cite{yuille2002concave}. The aforementioned formulation is:
	\begin{align}
	\label{eq:cccp}
		\hns \min_{\vw} \frac{1}{2} \NormII{\vw}^2 + C \hspace{-0.1in} \sum_{(x,y) \in S} \max_{(\yh, \hh) \in \Y_x \times \H_x}  \hspace{-0.05in} [\dotprod{\Phi(x,\yh,\hh)}{\vw} + d(y,\yh,\hh)] - C \hspace{-0.1in} \sum_{(x,y) \in S}  \max_{h \in \H_x} \dotprod{\Phi(x,y,h)}{\vw}
	\end{align}
	In the case of standard SSVMs (without latent variables), \cite{tsochantaridis2005large} discuss two advantages of the \textit{slack re-scaling} formulation over the margin re-scaling formulation, these are: the slack re-scaling formulation is invariant to the scaling of the distortion function, and the margin re-scaling potentially gives significant score to structures that are not even close to being confusable with the target structures. 
	\cite{Altun03,Collins04,Taskar03} proposed similar formulations to the slack re-scaling formulation. 
	Despite its theoretical advantages, the slack re-scaling has been less popular than the margin re-scaling approach due to computational requirements. 
	In particular, both formulations require optimizing over the output space, but while margin re-scaling preserves the structure of the score and error functions, the slack re-scaling does not.
	This results in harder inference problems during training. 
	\cite{honorio2016} also analyze the slack re-scaling approach and theoretically show that using random structures one can obtain a tighter upper bound of the Gibbs decoder distortion.
	However, these works do not take into account latent variables.
	
	The following formulation corresponds to the slack re-scaling approach with latent variables:
	\begin{align}
	\label{eq:slack_all}
		\min_{\vw} \frac{1}{n} \sum_{(x,y) \in S} \max_{(\yh, \hh) \in \Y_x \times \H_x}  \hspace{-0.05in} d(y,\yh,\hh) {\rm\ } \Ind{m(x,y,\yh,\hh,\vw) \leq 1} + {\lambda}\NormII{\vw}^2
	\end{align}
	We take into account the loss of structures whose margin is less than one (i.e., $m(\cdot) \leq 1$) instead of the Hamming distance as done in \cite{honorio2016}. 
	This is because the former gave better results in preliminary experiments. Also, it is more related to current practice (e.g., \cite{Yu09}).
	In order to obtain an SSVM-like formulation, the hinge loss is used instead of the discontinuous $0/1$ loss in the above formulation. 
	Note however, that both eq.\eqref{eq:cccp} and eq.\eqref{eq:slack_all} are now non-convex problems with respect to the learning parameter $\vw$ even if the hinge loss is used. 

%%%%%%%%%%%%%%%%%%%%%%%%%%%%%%%%%%%%%%%%%%%%%%%%%%%%%%%%%%%%%%%%%%%%%%%%%%%%%%%%%%%%%%%%%%%%%%%%%%%%%%%%%
%%%%%%%%%%%%%%%	MAX LOSS OF ALL STRUCTURED OUTPUTS
%%%%%%%%%%%%%%%%%%%%%%%%%%%%%%%%%%%%%%%%%%%%%%%%%%%%%%%%%%%%%%%%%%%%%%%%%%%%%%%%%%%%%%%%%%%%%%%%%%%%%%%%%

\section{The maximum loss over all structured outputs and latent variables}
\label{sec:pacbayes_all}

	In this section we extend the work of \citet{McAllester07} by including latent variables. In the following theorem, we show that the slack re-scaling objective function (eq.\eqref{eq:slack_all}) is an upper bound of the Gibbs decoder distortion (eq.\eqref{eq:gibbs_distortion}) up to an statistical accuracy of $\BigO{\sqrt{\nicefrac{\log n}{n}}}$ for $n$ training samples.
	
	\begin{theorem} \label{thrm:pacbayesall}
		Assume that there exists a finite integer value $r$ such that $|\Y_x \times \H_x| \leq r$ for all $(x,y) \in S$.
		Assume also that $\NormII{\Phi(x,y,h)} \leq \gamma$ for any triple $(x,y,h)$.
		Fix ${\delta \in (0,1)}$.
		With probability at least ${1-\delta/2}$ over the choice of $n$ training samples, simultaneously for all parameters ${\vw \in \W}$ and unit-variance Gaussian perturbation distributions ${Q(\vw)}$ centered at ${\vw \gamma\sqrt{8\log{(r n/\norm{\vw}_2^2)}}}$, we have:
		\begin{align*}
		L(Q(\vw),D)  \leq &\frac{1}{n}  \sum_{(x,y) \in S}{ \max_{ (\yh,\hh) \in \Y_x \times \H_x }{ \hspace{-0.05in} d(y,\yh,\hh) {\rm\ } \Ind{m(x,y,\yh,\hh,\vw) \leq 1} }} + \frac{\norm{\vw}_2^2}{n} \\
		&+ \sqrt{\frac{4\norm{\vw}_2^2 \gamma^2 \log{(r n /\norm{\vw}_2^2)} + \log{(2n/\delta)}}{2(n-1)}}
		\end{align*}
	\end{theorem}
	
	(See Appendix \ref{appendix} for detailed proofs.)
	
	For the proof of the above we used the PAC-Bayes theorem and well-known Gaussian concentration inequalities. Note that the average sum in the right-hand side can be equivalently written as:
	\[
	\frac{1}{n}  \sum_{(x,y) \in S}{ \max_{ (\yh,\hh) \in \Y_x \times \H_x } \min_{h \in \H_x} { d(y,\yh,\hh) {\rm\ } \Ind{\dotprod{\Phi(x,y,h)}{\vw} - \dotprod{\Phi(x,\yh,\hh)}{\vw} \leq 1} }}.
	\]
	\begin{remark}
		It is clear that the above formulation is tight with respect to the latent space $\H_x$ due to the minimization. This is an interesting observation because it reinforces the idea that a non-convex formulation is required in models using latent variables, i.e., an attempt to ``convexify'' the formulation will result in looser upper bounds and consequently might produce worse predictions. Some examples of non-convex formulations for latent-variable models are \cite{Yu09,hinton2012practical}. 
	\end{remark}
	Note also that the upper bound has a maximization over $\Y_x \times \H_x$ (usually exponential in size) and a minimization over $\H_x$. 
	In the minimization, it is clear that if one uses a subset of $\H_x$ instead of the whole latent space, this would lead to a looser upper bound. 
	In contrast, using a subset of $\Y_x \times \H_x$ in the maximization will lead to a tighter upper bound. 
	It is then natural to ask what elements should constitute this subset in order to control the statistical accuracy with respect to the Gibbs decoder. 
	Finally, if the number of elements is polynomial then we also have an efficient computation of the maximum. 
	We provide answers to these questions in the next section.

%%%%%%%%%%%%%%%%%%%%%%%%%%%%%%%%%%%%%%%%%%%%%%%%%%%%%%%%%%%%%%%%%%%%%%%%%%%%%%%%%%%%%%%%%%%%%%%%%%%%%%%%%
%%%%%%%%%%%%%%%	MAX LOSS OF RANDOM STRUCTURED OUTPUTS
%%%%%%%%%%%%%%%%%%%%%%%%%%%%%%%%%%%%%%%%%%%%%%%%%%%%%%%%%%%%%%%%%%%%%%%%%%%%%%%%%%%%%%%%%%%%%%%%%%%%%%%%%

\section{The maximum loss over random structured outputs and latent variables}
\label{sec:pacbayes_random}

	In this section, we show the relation between PAC-Bayes bounds and the maximum loss over random structured outputs and latent variables sampled \iid from some proposal distribution.
	
	\paragraph{A more efficient evaluation.} 
		Instead of using a maximization over ${\Y_x \times \H_x}$, we will perform a maximization over a set ${T(\vw,x)}$ of random elements sampled \iid from some proposal distribution ${R(\vw,x)}$ with support on ${\Y_x \times \H_x}$.
		More explicitly, our new formulation is:
		\begin{align}
		\label{eq:slack_random}
			\min_{\vw} \frac{1}{n} \sum_{(x,y) \in S} \max_{(\yh,\hh) \in T(\vw,x)}  \hspace{-0.05in} d(y,\yh,\hh) {\rm\ } \Ind{m(x,y,\yh,\hh,\vw) \leq 1} + \lambda \NormII{\vw}^2.
		\end{align}
		We make use of the following two assumptions in order for ${|T(\vw,x)|}$ to be polynomial, even when ${|\Y_x \times \H_x|}$ is exponential with respect to the input size.\footnote{Note that in order for the evaluation to be fully polynomial, the calculation of the margin has to be in polynomial time too, which is the case when the size of the latent space is polynomial or when there is an efficient way to compute the maximum over the latent space.}
		\begin{assumption}[Maximal distortion, \cite{honorio2016}] \label{asm:maxdistortion}
			The proposal distribution ${R(\vw,x)}$ fulfills the following condition.
			There exists a value ${\beta \in [0,1)}$ such that for all ${(x,y) \in S}$ and ${\vw \in \W}$:
			\begin{align*}
				\P_{(\yp,\hp) \sim R(\vw,x)}[d(y,\yp,\hp)=1] \geq 1-\beta
			\end{align*}
		\end{assumption}
		\begin{assumption}[Low norm] \label{asm:lownorm}
			The proposal distribution ${R(\vw,x)}$ fulfills the condition for all ${(x,y) \in S}$ and ${\vw \in \W}$:\footnote{The second inequality follows from an implicit assumption made in Theorem \ref{thrm:pacbayesall}, i.e., ${\norm{\vw}_2^2/n \leq 1}$ since the distortion function $d$ is at most $1$.}
			\begin{align*}
			\ANormII{\E_{(\yp,\hp) \sim R(\vw,x)} \left[ \Phi(x,y,\hs) - \Phi(x,\yp,\hp) \right] } &\leq \frac{1}{2 \sqrt{n}} \leq \frac{1}{2 \norm{\vw}_2},
			\end{align*}
			where $\hs = \argmax_{h \in \H_x} \Phi(x,y,h)\cdot \vw$.
		\end{assumption}
	
		In Section \ref{sec:examples} we provide examples for Assumptions \ref{asm:maxdistortion} and \ref{asm:lownorm} which allow us to obtain ${|T(\vw,x)| = \O\Big( \max{ \left( \frac{1}{\log{(1/\beta)}}, \gamma^2\NormII{\vw}^2\right)} \Big) }$.
		Note that $\beta$ plays an important role in the number of samples that we need to draw from the proposal distribution ${R(\vw,x)}$.
		
	\paragraph{Statistical analysis.}
		In this approach, randomness comes from two sources, from the training data $S$ and the random set $T(\vw,x)$.
		That is, in Theorem \ref{thrm:pacbayesall}, randomness only stems from the training set $S$.
		Now we need to produce generalization results that hold for all the sets $T(\vw,x)$, and for all possible proposal distributions ${R(\vw,x)}$.
		The following assumption will allow us to upper-bound the number of possible proposal distributions ${R(\vw,x)}$.
		\begin{assumption}[Linearly inducible ordering, \cite{honorio2016}] \label{asm:linearordering}
			The proposal distribution ${R(\vw,x)}$ depends solely on the linear ordering induced by the parameter ${\vw \in \W}$ and the mapping ${\Phi(x,\cdot,\cdot)}$.
			More formally, let ${r(x) \equiv |\Y_x\times \H_x|}$ and thus ${\Y_x \times \H_x \equiv \{(y_1,h_1) \dots (y_{r(x)},h_{r(x)})  \}}$.
			Let ${\vw,\vw' \in \W}$ be any two arbitrary parameters.
			Let ${\pi(x) = (\pi_1 \dots \pi_{r(x)})}$ be a permutation of ${\{1 \dots r(x)\}}$ such that ${\dotprod{\Phi(x,y_{\pi_1},h_{\pi_1})}{\vw} < \dots < \dotprod{\Phi(x,y_{\pi_{r(x)}},h_{\pi_{r(x)}})}{\vw}}$.
			Let ${\pip(x) = (\pip_1 \dots \pip_{r(x)})}$ be a permutation of ${\{1 \dots r(x)\}}$ such that ${\dotprod{\Phi(x,y_{\pip_1},h_{\pip_1})}{\vw'} < \dots < \dotprod{\Phi(x,y_{\pip_{r(x)}},h_{\pip_{r(x)}}  )}{\vw'}}$.
			For all ${\vw,\vw' \in \W}$ and ${x \in \X}$, if ${\pi(x)=\pip(x)}$ then ${\KL{R(\vw,x)}{R(\vw',x)}=0}$.
			In this case, we say that the proposal distribution fulfills ${R(\pi(x),x) \equiv R(\vw,x)}$.
		\end{assumption}
		In Assumption \ref{asm:linearordering}, geometrically speaking, for a fixed $x$ we first project the feature vectors ${\Phi(x,y,h)}$ of all ${(y,h) \in \Y_x \times \H_x}$ onto the lines $\vw$ and $\vw'$.
		Let ${\pi(x)}$ and ${\pip(x)}$ be the resulting ordering of the structured outputs after projecting them onto $\vw$ and $\vw'$ respectively.
		Two proposal distributions ${R(\vw,x)}$ and ${R(\vw',x)}$ are the same provided that ${\pi(x) = \pip(x)}$.
		That is, the specific values of ${\dotprod{\Phi(x,y,h)}{\vw}}$ and ${\dotprod{\Phi(x,y,h)}{\vw'}}$ are irrelevant, and only their ordering matters.
		
		In Section \ref{sec:examples} we show an example that fulfills Assumption \ref{asm:linearordering}, which corresponds to a generalization of Algorithm 2 proposed in \cite{honorio2016} for any structure with computationally efficient local changes.
		
	In the following theorem, we show that our new formulation in eq.\eqref{eq:slack_random} is related to an upper bound of the Gibbs decoder distortion up to statistical accuracy of ${\O(\nicefrac{\log^{2}{n}}{\sqrt{n}})}$ for $n$ training samples.
	\begin{theorem} \label{thm:pacbayesrandom}
		Assume that there exist finite integer values $r$, $\ell$, and $\gamma$ such that ${|\Y_x \times \H_x| \leq r}$ for all ${(x,y) \in S}$, $|\cup_{(x,y) \in S} \PS_x| \leq \ell$, and $\NormII{\Phi(x,y,h)} \leq \gamma$ for any triple $(x,y,h)$.
		Assume that the proposal distribution ${R(\vw,x)}$ with support on ${\Y_x \times \H_x}$ fulfills Assumption \ref{asm:maxdistortion} with value $\beta$, as well as Assumptions \ref{asm:lownorm} and \ref{asm:linearordering}.
		Fix ${\delta \in (0,1)}$ and an integer $\s$ such that ${3 \leq 2\s+1 \leq \frac{9}{20} \sqrt{\ell(r+1)+1}}$.
		With probability at least ${1-\delta}$ over the choice of both $n$ training samples and $n$ sets of random structured outputs and latent variables, simultaneously for all parameters ${\vw \in \W}$ with ${\norm{\vw}_0 \leq \s}$, unit-variance Gaussian perturbation distributions ${Q(\vw)}$ centered at $\vw \gamma\sqrt{8\log{(r n/\norm{\vw}_2^2)}}$, and for sets of random structured outputs ${T(\vw,x)}$ sampled \iid from the proposal distribution ${R(\vw,x)}$ for each training sample ${(x,y) \in S}$, such that ${|T(\vw,x)| = \ceil{\frac{1}{2} \max{\left(\frac{1}{\log{(1/\beta)}}, 128\gamma^2\NormII{\vw}^2 \right)} \log{n}}}$, we have:
		\begin{align*}
			L(Q(\vw),D) &\leq \frac{1}{n}  \sum_{(x,y) \in S}{ \max_{(\yh,\hh) \in T(\vw,x)}{ \hspace{-0.15in}d(y,\yh,\hh) {\rm\ } \Ind{ m(x,y,\yh,\hh,\vw) \leq 1} }} + \frac{\norm{\vw}_2^2}{n}  \\
			&+ \sqrt{\frac{4\norm{\vw}_2^2 \gamma^2 \log{\frac{rn}{\norm{\vw}_2^2}}  +  \log{\frac{2n}{\delta}} }{2(n-1)}}  +  \sqrt{\frac{1}{n}} \hspace{0.02in} + 3 \sqrt{\frac{\s (\log{\ell}+2 \log{(nr)})+\log{(4/\delta)}}{n}}\\
			&+ 2 \ {\textstyle \max{\hns \Big(\frac{1}{\log{(1/\beta)}}, 128\gamma^2\NormII{\vw}^2\Big)} } \sqrt{ \frac{(2\s + 1) \log(\ell(nr + 1) + 1) \log^3(n + 1)}{n}} \\
		\end{align*}
	\end{theorem}
	The proof of the above is based on Theorem \ref{thrm:pacbayesall} as a starting point.
	In order to account for the computational aspect of requiring sets ${T(\vw,x)}$ of polynomial size, we use Assumptions \ref{asm:maxdistortion} and \ref{asm:lownorm} for bounding a \emph{deterministic} expectation.
	In order to account for the statistical aspects, we use Assumption \ref{asm:linearordering} and Rademacher complexity arguments for bounding a \emph{stochastic} quantity for all sets ${T(\vw,x)}$ of random structured outputs and latent variables, and all possible proposal distributions ${R(\vw,x)}$.
	\begin{remark}
		A straightforward application of Rademacher complexity in the analysis of \cite{honorio2016} leads to a bound of $\BigO{|\H_x|/\sqrt{n}}$.
		Technically speaking, a classical Rademacher complexity states that: let $\Fcal$ and $\Gcal$ be two hypothesis classes. 
		Let $\min(\Fcal,\Gcal) = \{\min(f,g) | f \in \Fcal, g \in \Gcal \}$. 
		Then $\Rademacher(\min(\Fcal,\Gcal)) \leq \Rademacher(\Fcal) + \Rademacher(\Gcal)$. 
		If we apply this, then Theorem \ref{thm:pacbayesrandom} would contain an $\BigO{|\H_x|/\sqrt{n}}$ term, or equivalently $\BigO{r/\sqrt{n}}$. 
		This would be prohibitive since $r$ is typically exponential size, and one would require a very large number of samples $n$ in order to have a useful bound, i.e., to make $\BigO{r/\sqrt{n}}$ close to zero.
		In the proof we provide a way to tighten the bound to $ \BigO{\sqrt{\log{|\H_x |/n}}} $.
	\end{remark}

%%%%%%%%%%%%%%%%%%%%%%%%%%%%%%%%%%%%%%%%%%%%%%%%%%%%%%%%%%%%%%%%%%%%%%%%%%%%%%%%%%%%%%%%%%%%%%%%%%%%%%%%%
%%%%%%%%%%%%%%%	EXAMPLES
%%%%%%%%%%%%%%%%%%%%%%%%%%%%%%%%%%%%%%%%%%%%%%%%%%%%%%%%%%%%%%%%%%%%%%%%%%%%%%%%%%%%%%%%%%%%%%%%%%%%%%%%%

\section{Examples} 
\label{sec:examples}

	Here we provide several examples that fulfill the three main assumptions of our theoretical result.
	
	\paragraph{Examples for Assumption \ref{asm:maxdistortion}.}
		First we argue that we can perform a change of measure between different proposal distributions.
		This allows us to focus on uniform proposals afterwards.
		
		\begin{claim}[Change of measure] \label{clm:changeofmeasure}
			Let ${R(\vw,x)}$ and ${\Rp(\vw,x)}$ two proposal distributions, both with support on ${\Y_x \times \H_x}$.
			Assume that ${R(\vw,x)}$ fulfills Assumption \ref{asm:maxdistortion} with value ${\beta_1}$.
			Let ${r_{\vw,x}(\cdot)}$ and ${\rp_{\vw,x}(\cdot)}$ be the probability mass functions of ${R(\vw,x)}$ and ${\Rp(\vw,x)}$ respectively.
			Assume that the total variation distance between ${R(\vw,x)}$ and ${\Rp(\vw,x)}$ fulfills for all ${(x,y) \in S}$ and ${\vw \in \W}$:
			\begin{align*}
				TV(R(\vw,x) \| \Rp(\vw,x)) & \equiv  \frac{1}{2} \sum_{(y,h)}{ |r_{\vw,x}(y,h) - \rp_{\vw,x}(y,h)|} \leq \beta_2
			\end{align*}
			Then ${\Rp(\vw,x)}$ fulfills Assumption \ref{asm:maxdistortion} with ${\beta = \beta_1 + \beta_2}$ provided that ${\beta_1 + \beta_2 \in [0,1)}$.
		\end{claim}
		
		Next, we present a new result for permutations and for a distortion that returns the number of different positions. 
		We later use this result for an image matching application in the experiments section.
		
		\begin{claim}[Permutations] \label{clm:permutations}
			Let ${\Y_x}$ be the set of all permutations of $v$ elements, such that ${v > 1}$.
			Let $y_i$ be the $i$-th element in the permutation $y$.
			Let ${d(y,\yp,h) = \frac{1}{v} \sum_{i=1}^v \Ind{y_i \neq \yp_i}}$.
			The uniform proposal distribution ${R(\vw,x) = R(x)}$ with support on ${\Y_x \times \H_x}$ fulfills Assumption \ref{asm:maxdistortion} with ${\beta = 2/3}$.
		\end{claim}
		
		The authors in \cite{honorio2016} present several examples of distortion functions of the form $d(y,\yp)$, for directed spanning trees, directed acyclic graphs and cardinality-constrained sets, and a distortion function that returns the number of different edges/elements; as well as for any type of structured output and binary distortion functions. 
		For our setting we can make use of these examples by defining $d(y,\yp,h) = d(y,\yp)$. 
		Note that even if we ignore the latent variable in the distortion function, we still use the latent variables in the feature vectors $\Phi(x,y,h)$ and thus in the calculation of the margin.

	\paragraph{Examples for Assumption \ref{asm:lownorm}.}
		The claim below is for a particular instance of a sparse mapping and a uniform proposal distribution.
		\begin{claim}[Sparse mapping] \label{clm:sparsedata}
			Let ${b>0}$ be an arbitrary integer value.
			For all ${(x,y) \in S}$ with $\hs = \argmax_{h \in \H_x} \Phi(x,y,h)\cdot\vw$, let ${\Upsilon_x = \cup_{p \in \PS_x}{\Upsilon_x^p}}$, where the partition ${\Upsilon_x^p}$ is defined as follows for all $p \in \PS_x$:
			\begin{align*}
				\Upsilon_x^p \equiv \{ & (\yp,\hp) \mid  |\Phi_p(x,y,\hs) - \Phi_p(x,\yp,\hp)| \leq b \; \ {\rm and} \ 
				 					    (\forall q \neq p) {\rm\ } \Phi_q(x,y,\hs) = \Phi_q(x,\yp,\hp) \}
			\end{align*}
			If ${n \leq |\PS_x|/(4b^2)}$ for all ${(x,y) \in S}$, then the uniform proposal distribution ${R(\vw,x) = R(x)}$ with support on ${\Y_x \times \H_x}$ fulfills Assumption \ref{asm:lownorm}.
		\end{claim}
		The claim below is for a particular instance of a dense mapping and an \emph{arbitrary} proposal distribution.
		\begin{claim}[Dense mapping] \label{clm:densedata}
			Let ${b>0}$ be an arbitrary integer value.
			Let ${|\Phi_p(x,y,\hs) - \Phi_p(x,\yp,\hp)| \leq \frac{b}{|\PS_x|}}$ for all ${(x,y) \in S}$ with $\hs = \argmax_{h \in \H_x} \Phi(x,y,h)\cdot\vw$, ${(\yp,\hp) \in \Y_x \times \H_x}$ and ${p \in \PS_x}$.
			If ${n \leq |\PS_x|/(4b^2)}$ for all ${(x,y) \in S}$, then any arbitrary proposal distribution ${R(\vw,x)}$ fulfills Assumption \ref{asm:lownorm}.
		\end{claim}

	\paragraph{Examples for Assumption \ref{asm:linearordering}.}
		In the case of modeling without latent variables, \cite{Zhang14,Zhang15} presented an algorithm for directed spanning trees in the context of dependency parsing in natural language processing. 
		Later, \cite{honorio2016} extended the previous algorithm to any structure with computationally efficient local changes, which includes directed acyclic graphs (traversed in post-order) and cardinality-constrained sets.
		Next, we generalize Algorithm 2 in \cite{honorio2016} by including latent variables.
		\begin{algorithm}[H]
			\small
			\caption{\small Procedure for sampling a structured output ${(\yp,\hp) \in \Y_x \times \H_x}$ from a greedy local proposal distribution ${R(\vw,x)}$}
			\label{alg:greedylocal}
				\begin{algorithmic}[1]
				\STATE \textbf{Input:} parameter $\vw \in \W$, observed input $x \in \X$
				\STATE Draw uniformly at random a structured output ${(\yh,\hh) \in \Y_x \times \H_x}$
				\REPEAT
				  \STATE Make a local change to $(\yh,\hh)$ in order to increase ${\dotprod{\Phi(x,\yh,\hh)}{\vw}}$
				\UNTIL no refinement in last iteration
				\STATE \textbf{Output:} structured output and latent variable ${(\yp,\hp) \leftarrow (\yh,\hh)}$
				\end{algorithmic}
		\end{algorithm}
		The above algorithm has the following property:
		\begin{claim}[Sampling for any type of structured output and latent variable] \label{clm:greedylocal}
			Algorithm \ref{alg:greedylocal} fulfills Assumption \ref{asm:linearordering}.
		\end{claim}

%%%%%%%%%%%%%%%%%%%%%%%%%%%%%%%%%%%%%%%%%%%%%%%%%%%%%%%%%%%%%%%%%%%%%%%%%%%%%%%%%%%%%%%%%%%%%%%%%%%%%%%%%
%%%%%%%%%%%%%%%	EXPERIMENTS
%%%%%%%%%%%%%%%%%%%%%%%%%%%%%%%%%%%%%%%%%%%%%%%%%%%%%%%%%%%%%%%%%%%%%%%%%%%%%%%%%%%%%%%%%%%%%%%%%%%%%%%%%

\section{Experiments}
\label{sec:experiments}

	In this section we illustrate the use of our approach by using the formulation in eq.\eqref{eq:slack_random}.
	The goal of the synthetic experiments is to show the improvement in prediction results and runtime of our method.
	While the goal of the real-world experiment is to show the usability of our method in practice.
	
	\paragraph{Synthetic experiments.}
		We present experimental results for directed spanning trees, directed acyclic graphs and cardinality-constrained sets.
		We performed $30$ repetitions of the following procedure.
		We generated a ground truth parameter ${\vw^*}$ with independent zero-mean and unit-variance Gaussian entries.
		Then, we generated a training set of ${n=100}$ samples.
		Our mapping ${\Phi(x,y,h)}$ is as follows.
		For every pair of possible edges/elements $i$ and $j$, we define ${\Phi_{ij}(x,y,h) = \Ind{ (h_{ij}\ \text{xor}\ x_{ij}) \ \text{and}\ i \in y \ \text{and}\ j \in y}}$.
		Here $x$ is a randomly generated binary string, $h$ corrects one bit of $x$, and $y$ is generated as the solution of eq.\eqref{eq:inferenceall}. (Details in Appendix \ref{app:experiments}.)
		
		We compared three training methods: the maximum loss over \textit{all} possible structured outputs and latent variables with slack re-scaling as in eq.\eqref{eq:slack_all} and with margin re-scaling as in eq.\eqref{eq:cccp} \cite{Yu09}.
		We also evaluate the maximum loss over \textit{random} structured outputs and latent variables as in eq.\eqref{eq:slack_random}.
		We considered directed spanning trees of $4$ nodes, directed acyclic graphs of $4$ nodes and $2$ parents per node, and sets of $3$ elements chosen from $9$ possible elements.
		After training, for inference on an independent test set, we used eq.\eqref{eq:inferenceall} for the maximum loss over \textit{all} possible structured outputs and latent variables.
		For the maximum loss over random structured outputs and latent variables, we use the following \emph{approximate} inference approach:
		\begin{align} \label{eq:inferencerandom}
		\ft_{\vw}(x) \equiv \argmax_{(y,h) \in T(\vw,x)}{\dotprod{\Phi(x,y,h)}{\vw}}
		\end{align}
		Table \ref{tab:results} shows the runtime, the training distortion as well as the test distortion in an independently generated set of $100$ samples.
		In the different study cases, the maximum loss over \textit{random} structured outputs and latent variables outperforms the maximum loss over \textit{all} possible structured outputs and latent variables.
		While the margin re-scaling approach \cite{Yu09} and our randomized approach obtain statistically-significantly similar test errors, our method is considerable faster.
	
		\begin{table*}[!tb]
			\caption{%
				\small
				Average over $30$ repetitions, and standard error at $95\%$ confidence level.
				\textit{All} indicates the use of exact learning and exact inference.
				\textit{Random} and \textit{Random/All} indicate use of random learning, and random and exact inference respectively.
				\textit{LSSVM} indicates the use of the method in \cite{Yu09}.
				Random/All outperforms All in the different study cases. 
				While LSSVM and Random/All obtain statistically-significantly similar performances on the test sets. 
				Note however that the runtime for learning using the randomized approach is much less than LSSVM and All.
			}
			
			\label{tab:results}
			\begin{center}
				\small
				\begin{tabular}{@{}l@{\hspace{0.125in}}l@{\hspace{0.125in}}c@{\hspace{0.125in}}c@{\hspace{0.125in}}c@{\hspace{0.125in}}c@{\hspace{0.125in}}c@{\hspace{0.125in}}c@{}}
					\hline
					\textbf{Problem} & \textbf{Method} & \textbf{Training} & \textbf{Training} & \textbf{Test} & \textbf{Test} \\
					 & & \textbf{runtime} & \textbf{distortion} & \textbf{runtime} & \textbf{distortion} \\
					\hline
					Directed & All & 1024 $\pm$ 5 & 22\% $\pm$ 2.5\% & 19.4 $\pm$ 0.3 & 21\% $\pm$ 2.6\% \\ %& 0.51 $\pm$ 0.013 & 63$^\circ$ $\pm$ 1.5$^\circ$ \\
					spanning trees & Random & \textbf{33 $\pm$ 0} & 30\% $\pm$ 1.9\% & 0.7 $\pm$ 0.0 & 30\% $\pm$ 1.5\% \\ %& 0.50 $\pm$ 0.013 & 62$^\circ$ $\pm$ 1.2$^\circ$ \\
					 & Random/All & & & 19.7 $\pm$ 0.3 & 15\% $\pm$ 1.6\% \\ %& & \\
					 & LSSVM & 1000 $\pm$ 0 & 16\% $\pm$ 2.7\% & 19.5 $\pm$ 0.3 & 15\% $\pm$ 2.8\% \\ %& 0.51 $\pm$ 0.012 & 62$^\circ$ $\pm$ 1.3$^\circ$ \\
					\hline
					Directed & All & 1024 $\pm$ 4 & 12\% $\pm$ 0.9\% & 19.4 $\pm$ 0.2 & 20\% $\pm$ 1.7\% \\ %& 0.39 $\pm$ 0.012 & 38$^\circ$ $\pm$ 2.7$^\circ$ \\
					acyclic graphs & Random & \textbf{51 $\pm$ 1} & 19\% $\pm$ 0.9\% & 1.1 $\pm$ 0.0 & 25\% $\pm$ 1.2\% \\ %& 0.36 $\pm$ 0.012 & 35$^\circ$ $\pm$ 1.9$^\circ$ \\
					 & Random/All& & & 19.5 $\pm$ 0.2 & 19\% $\pm$ 1.5\% \\ %& & \\
					 & LSSVM & 1000 $\pm$ 0 & 11\% $\pm$ 1.1\% & 19.4 $\pm$ 0.2 & 18\% $\pm$ 1.8\% \\ %& 0.37 $\pm$ 0.013 & 35$^\circ$ $\pm$ 2.6$^\circ$ \\
					\hline
					Cardinality & All & 1020 $\pm$ 4 & 19\% $\pm$ 2.7\% & 19.4 $\pm$ 0.3 & 20\% $\pm$ 3.1\% \\ %& 0.48 $\pm$ 0.017 & 55$^\circ$ $\pm$ 1.9$^\circ$ \\
					constrained sets & Random & \textbf{51 $\pm$ 0} & 29\% $\pm$ 1.5\% & 1.1 $\pm$ 0.1 & 31\% $\pm$ 1.7\% \\ %& 0.47 $\pm$ 0.016 & 53$^\circ$ $\pm$ 1.3$^\circ$ \\
					 & Random/All& & & 19.3 $\pm$ 0.3 & 16\% $\pm$ 2.1\% \\ %& & \\
					 & LSSVM & 1000 $\pm$ 0 & 14\% $\pm$ 2.5\% & 19.4 $\pm$ 0.2 & 15\% $\pm$ 2.8\% \\ %& 0.47 $\pm$ 0.017 & 52$^\circ$ $\pm$ 1.7$^\circ$ \\
					\hline
				\end{tabular}
			\end{center}
		\end{table*}
	
	\paragraph{Image matching.}
		We illustrate our approach for image matching on video frames from the Buffy Stickmen dataset (\url{http://www.robots.ox.ac.uk/~vgg/data/stickmen/}).
		The goal of the experiment is to match the keypoints representing different body parts, between two images.
		Each frame contains 18 keypoints representing different parts of the body.
		From a total of 187 image pairs (from different episodes and people), we randomly selected 120 pairs for training and the remaining 67 pairs for testing.
		We performed 30 repetitions.
		Ground truth keypoint matching is provided in the dataset.
		
		Following \cite{gane2014learning,volkovs2012efficient}, we represent the matching as a permutation of keypoints. 
		Let $x=(I,I')$ be a pair of images, and let $y$ be a permutation of $\{1\ldots 18\}$.
		%\Red{The training set $S$ consists of the }
		We model the latent variable $h$ as a $\R^{2\times 2}$ matrix representing an affine transformation of a keypoint, where $h_{11}, h_{22} \in \{0.8, 1, 1.2\}$, and $h_{12}, h_{21} \in \{-0.2, 0, 0.2\}$.
		Our mapping $\Phi(x,y,h)$ uses SIFT features, and the distance between coordinates after using $h$. (Details in Appendix \ref{app:experiments}.)
		
		We used the distortion function and $\beta = \nicefrac{2}{3}$ as prescribed by Claim \ref{clm:permutations}.
		After learning, for a given $x$ from the test set, we performed 100 iterations of random inference as in eq.\eqref{eq:inferencerandom}.
		We obtained an average error of $0.3878$ (6.98 incorrectly matched keypoints) in the test set, which is an improvement to the value of $8.69$ as reported in \cite{gane2014learning}.
		Finally, we show an example from the test set in Figure \ref{fig:buffy}.
		
		\begin{figure}[!ht]
			\centering
			\includegraphics[width=0.75\textwidth]{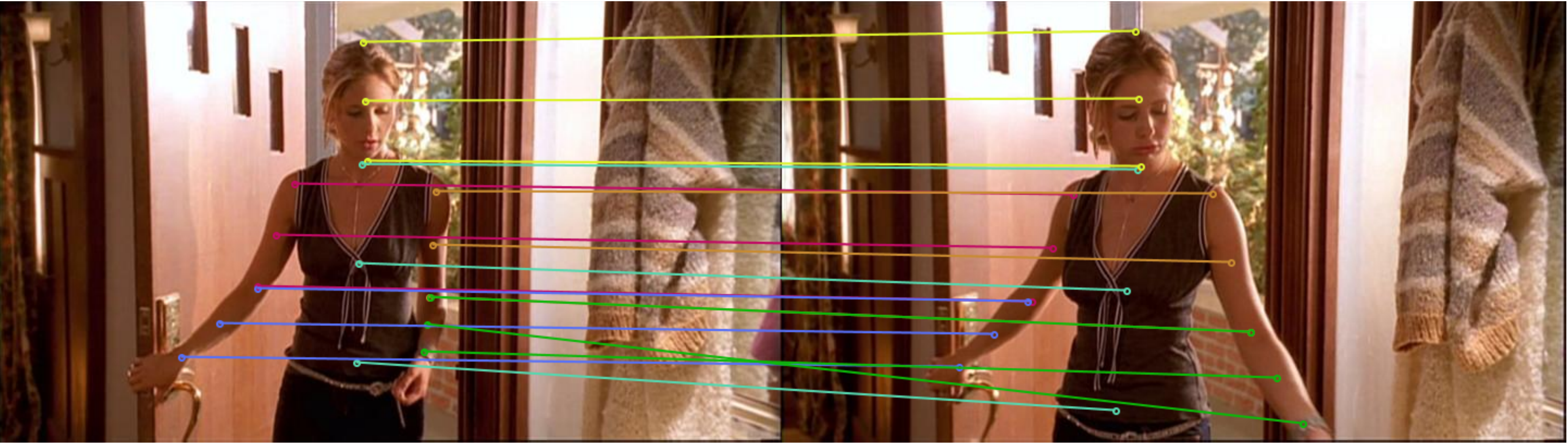}
			\caption{\small Image matching on the Buffy Stickmen dataset, predicted by our randomized approach with latent variables. The problem is challenging since the dataset contains different episodes and people.}
			\label{fig:buffy}
		\end{figure}

%%%%%%%%%%%%%%%%%%%%%%%%%%%%%%%%%%%%%%%%%%%%%%%%%%%%%%%%%%%%%%%%%%%%%%%%%%%%%%%%%%%%%%%%%%%%%%%%%%%%%%%%%
%%%%%%%%%%%%%%%	CONCLUDING REMARKS
%%%%%%%%%%%%%%%%%%%%%%%%%%%%%%%%%%%%%%%%%%%%%%%%%%%%%%%%%%%%%%%%%%%%%%%%%%%%%%%%%%%%%%%%%%%%%%%%%%%%%%%%%

\paragraph{Future directions.}
	The randomization of the latent space in the calculation of the margin is of high interest. 
	Despite leading to a looser upper bound of the Gibbs decoder distortion, if one could control the statistical accuracy under this approach then one could obtain a fully polynomial-time evaluation of the objective function, even if $|\H|$ is exponential.
	Therefore, whether this method is feasible, and under what technical conditions, are potential future work.
	The analysis of other non-Gaussian perturbation models from the computational and statistical viewpoints is also of interest.
	Finally, it would be interesting to analyze \emph{approximate} inference for prediction on an independent test set.

%%%%%%%%%%%%%%%%%%%%%%%%%%%%%%%%%%%%%%%%%%%%%%%%%%%%%%%%%%%%%%%%%%%%%%%%%%%%%%%%%%%%%%%%%%%%%%%%%%%%%%%%%
%%%%%%%%%%%%%%%	REFERENCES AND APPENDIX
%%%%%%%%%%%%%%%%%%%%%%%%%%%%%%%%%%%%%%%%%%%%%%%%%%%%%%%%%%%%%%%%%%%%%%%%%%%%%%%%%%%%%%%%%%%%%%%%%%%%%%%%%

%\small
\bibliographystyle{agsm}
\bibliography{paper}

\normalsize
\clearpage
%\newgeometry{top=1in,bottom=1in,left=0.95in,right=0.8in}
\begin{appendices}
%\onecolumn

%
\def\toptitlebar{
	\hrule height4pt
	\vskip .25in}

\def\bottomtitlebar{
	\vskip .25in
	\hrule height1pt
	\vskip .25in}

\thispagestyle{empty}
\hsize\textwidth
\linewidth\hsize \toptitlebar {\centering
{\large\bf SUPPLEMENTARY MATERIAL \\ Learning Latent Variable Structured Prediction Models with Gaussian Perturbations \par}}
\vspace{-0.1in} \bottomtitlebar

%%%%%%%%%%%%%%%%%%%%%%%%%%%%%%%%%%%%%%%%%%%%%%%%%%%%%%%%%%%%%%%%%%%%%%%%%%%%%%%%%%%%%%%%%%%%%%%%%%%%%%%%%%%%
%%%%%%%%%%%%%%%%%%%%%%%%%%%%%%%%%%%%%%%%%%%%%%%%%%%%%%%%%%%%%%%%%%%%%%%%%%%%%%%%%%%%%%%%%%%%%%%%%%%%%%%%%%%%
%%%%%%%%%%%% PROOFS
%%%%%%%%%%%%%%%%%%%%%%%%%%%%%%%%%%%%%%%%%%%%%%%%%%%%%%%%%%%%%%%%%%%%%%%%%%%%%%%%%%%%%%%%%%%%%%%%%%%%%%%%%%%%
%%%%%%%%%%%%%%%%%%%%%%%%%%%%%%%%%%%%%%%%%%%%%%%%%%%%%%%%%%%%%%%%%%%%%%%%%%%%%%%%%%%%%%%%%%%%%%%%%%%%%%%%%%%%

\section{Detailed Proofs}
\label{appendix}

In this section, we state the proofs of all the theorems in our manuscript.

\subsection{Proof of Theorem \ref{thrm:pacbayesall}}
\label{app:proof_all}

Here, we provide the proof of Theorem \ref{thrm:pacbayesall}.
First, we derive an intermediate lemma needed for the final proof.

\begin{lemma}[Adapted from Lemma~5 in \cite{McAllester07}] \label{lem:gaussian}
	Assume that there exists a finite integer value $r$ such that, $|\Y_x \times \H_x| \leq r$ for all $(x,y) \in S$.
	Assume also that $\NormII{\Phi(x,y,h)} \leq \gamma$ for any triple $(x,y,h)$.
	Let ${Q(\vw)}$ be a unit-variance Gaussian distribution centered at ${\alpha \vw}$ for $\alpha = \gamma\sqrt{8\log  \frac{rn}{\NormII{\vw}^2} }$.
	Then for all ${(x,y) \in S}$, and all ${\vw \in \W}$, we have:
	\begin{align*}
		\P_{\vw' \sim Q(\vw)}[   m(x,y, \langle f_{\vw'}(x) \rangle, \vw) \geq 1 ] \leq \NormII{\vw}^2/n
	\end{align*}
	\noindent or equivalently:
	\begin{align} \label{eq:gaussianwhp}
		\P_{\vw' \sim Q(\vw)}[ m(x,y, \langle f_{\vw'}(x)\rangle, \vw) \leq 1] \geq 1 - \NormII{\vw}^2/n
	\end{align}
\end{lemma}

\begin{proof}
	Note that the randomness in the statement comes from the variable $\vw'$, then by a union bound on the elements of $\Y_x \times \H_x$ it suffices to show that for any given $(\yh, \hh)$ with $m(x,y,\yh,\hh,\vw) \geq 1$, the probability that $ f_{\vw'} (x) = (\yh, \hh)$ is at most $\NormII{\vw}^2/(rn)$. 
	
	Consider a fixed $(\yh, \hh) \in \Y_x \times \H_x$ with $m(x,y,\yh,\hh,\vw) \geq 1$. First, by well-know concentration inequalities we have that for any vector $\Psi \in \R^\ell$ with $\norm{\Psi}_2 = 1$ and $\eps \geq 0$:
	\begin{equation}
	\label{eq:lemma1_1}
	\P_{\vw' \sim Q(\vw)}[ (\alpha \vw - \vw') \cdot \Psi \geq \eps ] \leq  \exp{-\eps^2/2} 
	\end{equation}
	Let $\hs = \argmax_{h \in \H_x} \Phi(x,y,h) \cdot \vw$, and let $\Delta(x,y,\hs, \yh, \hh) = \Phi(x,y,\hs) - \Phi(x,\yh,\hh)$. Then, $m(x,y,\yh,\hh,\vw) = \Delta(x,y,\hs, \yh, \hh) \cdot \vw$. Using $\Psi = \Delta(x,y,\hs, \yh, \hh) / \lVert \Delta(x,y,\hs, \yh, \hh) \rVert_2$ in \eqref{eq:lemma1_1} we have:
	\begin{align} \refstepcounter{equation}
		\P_{\vw' \sim Q(\vw)}[  m(x,y,\yh,\hh,\vw') \leq \alpha m(x,y,\yh,\hh,\vw) - \eps \lVert \Delta(x,y,\hs, \yh, \hh) \rVert_2 ] &\leq  \exp{-\eps^2/2}  \nonumber \\
		\P_{\vw' \sim Q(\vw)}[  m(x,y,\yh,\hh,\vw') \leq \alpha  - \eps \lVert \Delta(x,y,\hs, \yh, \hh) \rVert_2 ] &\leq  \exp{-\eps^2/2} \nonumber \\
		\P_{\vw' \sim Q(\vw)}[  m(x,y,\yh,\hh,\vw') \leq 0 ] &\leq  \exp{-\alpha^2/(8\gamma^2)} \tag{\theequation.a} \label{eq:lemma_1_2}\\
		\P_{\vw' \sim Q(\vw)}[   f_{\wp} (x) = (\yh, \hh) ] &\leq  \exp{-\alpha^2/(8\gamma^2)} \nonumber
	\end{align}
	where the step in \eqref{eq:lemma_1_2} follows from $\eps = \alpha / \lVert \Delta(x,y,\hs, \yh, \hh) \rVert_2 $ and $\lVert \Delta(x,y,\hs, \yh, \hh) \rVert_2\leq 2\gamma$.
	Thus, we prove our claim.
\end{proof}

Next, we provide the final proof.

\begin{proof}[Proof of Theorem \ref{thrm:pacbayesall}]
	Define the Gibbs decoder \emph{empirical} distortion of the perturbation distribution ${Q(\vw)}$ and training set $S$ as:
	\begin{align*}
	L(Q(\vw),S) = \frac{1}{n} \sum_{(x,y) \in S}{\E_{\vw' \sim Q(\vw)}[d(y, \Inn{f_{\vw'}(x)})]}
	\end{align*}
	In PAC-Bayes terminology, ${Q(\vw)}$ is the \emph{posterior} distribution.
	Let the \emph{prior} distribution $P$ be the unit-variance zero-mean Gaussian distribution.
	Fix ${\delta \in (0,1)}$ and ${\alpha>0}$.
	By well-known PAC-Bayes proof techniques, Lemma~4 in \cite{McAllester07} shows that with probability at least ${1-\delta/2}$ over the choice of $n$ training samples, simultaneously for all parameters ${\vw \in \W}$, and unit-variance Gaussian posterior distributions ${Q(\vw)}$ centered at ${\vw \alpha}$, we have:
	\begin{align} \label{eq:pacbayes}
	L(Q(\vw),D) & \leq L(Q(\vw),S) + \sqrt{\frac{KL(Q(\vw) \| P) + \log{(2n/\delta)}}{2(n-1)}} \nonumber \\
	& = L(Q(\vw),S) + \sqrt{\frac{\norm{w}_2^2 \alpha^2/2 + \log{(2n/\delta)}}{2(n-1)}}
	\end{align}
	Thus, an upper bound of ${L(Q(\vw),S)}$ would lead to an upper bound of ${L(Q(\vw),D)}$.
	In order to upper-bound ${L(Q(\vw),S)}$, we can upper-bound each of its summands, i.e., we can upper-bound ${\E_{\vw' \sim Q(\vw)}[d(y, f_{\vw'}(x))]}$ for each ${(x,y) \in S}$.
	Define the distribution ${Q(\vw,x)}$ with support on ${\Y_x \times \H_x}$ in the following form for all ${y \in \Y_x}$ and $h \in \H_x$:
	\begin{align} \label{eq:Qwx}
	\P_{ (\yp,\hp) \sim Q(\vw,x)}[(\yp,\hp) = (y,h)] \equiv \P_{\vw' \sim Q(\vw)}[f_{\vw'}(x) = (y,h)]
	\end{align}
	For clarity of presentation, define:
	\begin{align*}
	u(x,y,\yp,\hp,\vw) \equiv 1 - m(x,y,\yp,\hp,\vw)
	\end{align*}
	Let ${u \equiv u(x,y, \Inn{f_{\vw'}(x)},\vw)}$.
	Simultaneously for all ${(x,y) \in S}$, we have:
	\begin{align} \refstepcounter{equation}
	\Exp{\vw' \sim Q(\vw)}{d(y,\Inn{f_{\vw'}(x)}} & = \Exp{\vw' \sim Q(\vw)}{d(y, \Inn{f_{\vw'}(x)}) {\rm\ } \Ind{u \geq 0} + d(y,\Inn{f_{\vw'}(x)}) {\rm\ } \Ind{u < 0}} \nonumber \\
	& \leq \Exp{\vw' \sim Q(\vw)}{d(y, \Inn{f_{\vw'}(x)}) {\rm\ } \Ind{u \geq 0} + \Ind{u < 0}} \tag{\theequation.a}\label{eq:dtimesiversontoiverson} \\
	& = \Exp{\vw' \sim Q(\vw)}{d(y,\Inn{f_{\vw'}(x)} {\rm\ } \Ind{u \geq 0}} + \P_{\vw' \sim Q(\vw)}[u < 0] \nonumber \\
	& \leq \Exp{\vw' \sim Q(\vw)}{d(y,\Inn{f_{\vw'}(x)} {\rm\ } \Ind{u \geq 0}} + \norm{\vw}_2^2/n \tag{\theequation.b}\label{eq:highprobHm} \\
	& = \Exp{\vw' \sim Q(\vw)}{ d(y, \Inn{f_{\vw'}(x)} {\rm\ } \Ind{u(x,y,\Inn{f_{\vw'}(x)},\vw) \geq 0} } + \norm{\vw}_2^2/n \nonumber \\
	& = \Exp{ (\yp,\hp) \sim Q(\vw,x) }{d(y,\yp,\hp) {\rm\ } \Ind{u(x,y,\yp,\hp,\vw) \geq 0}} + \norm{\vw}_2^2/n \tag{\theequation.c}\label{eq:QwtoQwx} \\
	& \leq \max_{(\yh,\hh) \in \Y_x \times \H_x}{d(y,\yh,\hh) {\rm\ } \Ind{u(x,y,\yh,\hh,\vw) \geq 0}} + \norm{\vw}_2^2/n \tag{\theequation.d}\label{eq:expectedtomax}
	\end{align}
	\noindent where the step in eq.\eqref{eq:dtimesiversontoiverson} holds since ${d : \Y \times \Y \times \H \to [0,1]}$.
	The step in eq.\eqref{eq:highprobHm} follows from Lemma \ref{lem:gaussian} which states that ${\P_{\vw' \sim Q(\vw)}[u(x,y,\Inn{f_{\vw'}(x)},\vw) < 0] \leq \norm{\vw}_2^2/n}$ for $\alpha = \gamma\sqrt{8\log{(rn/\norm{\vw}_2^2)}}$, for all ${(x,y) \in S}$ and all $\vw \in \W$.
	By the definition in eq.\eqref{eq:Qwx}, then the step in eq.\eqref{eq:QwtoQwx} holds.
	Let ${\lambda : \Y \times \H \to [0,1]}$ be some arbitrary function, the step in eq.\eqref{eq:expectedtomax} uses the fact that ${\E_{(y,h)}[\lambda(y,h)] \leq \max_{(y,h)}{\lambda(y,h)}}$.
	
	By eq.\eqref{eq:pacbayes} and eq.\eqref{eq:expectedtomax}, we prove our claim.
	\qedhere
\end{proof}

%%%%%%%%%%%%%%%%%%%%%%%%%%%%%%%%
%%%%%%%%%%%%%%%%%%%%%%%%%%%%%%%%
%%%%%%%%%%%%%%%%%%%%%%%%%%%%%%%%

\subsection{Proof of Theorem \ref{thm:pacbayesrandom}}
\label{app:proof_random}

Here, we provide the proof of Theorem \ref{thm:pacbayesrandom}.
First, we derive an intermediate lemma needed for the final proof.

\begin{lemma} \label{lem:phmbound}
	Let ${\Delta \in \R^\ell}$ be a random variable with $\NormII{\Delta} \leq 2\gamma$, and ${\vw \in \R^\ell}$ be a constant.
	If $\dotprod{\E[\Delta]}{\vw} \leq 1/2  $ then we have:
%	If $2s \leq \frac{1}{2\NormII{\vw}}$ then we have:
	%
	\begin{align*}
	\P[  \dotprod{\Delta}{\vw} > 1] \leq \lexp{\frac{-1}{128\gamma^2\NormII{\vw}^2}}
	\end{align*}
\end{lemma}

\begin{proof}
	Let ${t > 0}$, we have that:
	\begin{align} \refstepcounter{equation}
	\P[ \dotprod{\Delta}{\vw} > 1] & =  \P[ \dotprod{(\Delta - \E[\Delta])}{\vw} > 1 - \dotprod{\E[ \Delta ]}{\vw}] \nonumber \\
	& \leq \P[\dotprod{(\Delta - \E[\Delta])}{\vw} \geq 1/2] \tag{\theequation.b}\label{eq:uselownorm} \\
	& = \P[\lexp{t \dotprod{(\Delta - \E[\Delta])}{\vw}} \geq \exp{t/2}] \nonumber \\
	& \leq \exp{-t/2} \, \E[\lexp{t \dotprod{(\Delta - \E[\Delta])}{\vw}}] \tag{\theequation.c}\label{eq:usemarkov} \\
	& \leq \lexp{-t/2 + 8t^2\gamma^2\NormII{\vw}^2} \tag{\theequation.d}\label{eq:usehoeffding}
	\end{align}
	%
%	By Jensen's inequality we have ${\norm{\E[\Delta]}_2 \leq \E[{\norm{\Delta}_2}] \leq 2s}$.
	The step in eq.\eqref{eq:uselownorm} follows from $\dotprod{\E[\Delta]}{\vw} \leq 1/2$ and thus ${1-\dotprod{\E[\Delta]}{\vw} \geq 1/2}$.
	The step in eq.\eqref{eq:usemarkov} follows from Markov's inequality.
	The step in eq.\eqref{eq:usehoeffding} follows from Hoeffding's lemma and the fact that the random variable ${z = \dotprod{(\Delta - \E[\Delta])}{\vw}}$ fulfills ${\E[z] = 0}$ as well as ${z \in [-4\gamma\NormII{\vw}, +4\gamma\NormII{\vw}]}$.
	In more detail, note that ${\NormII{\Delta} \leq 2\gamma}$ and by Jensen’s inequality $ \NormII{\E[\Delta]} \leq \E[\NormII{\Delta}] \leq 2\gamma$.
	Then, note that by Cauchy-Schwarz inequality ${|\dotprod{(\Delta - \E[\Delta])}{\vw}| \leq \norm{\Delta - \E[\Delta]}_2 \norm{\vw}_2 \leq (\norm{\Delta}_2 + \norm{\E[\Delta]}_2) \norm{\vw}_2 \leq 4\gamma \norm{\vw}_2}$.
	Finally, let $g(t) = -t/2 + 8t^2\gamma^2\NormII{\vw}^2$.
	By making ${\partial g / \partial t = 0}$, we get the optimal setting ${t^* = 1/(32\gamma^2\NormII{\vw}^2)}$.
	Thus, ${g(t^*) = -1/(128\gamma^2\NormII{\vw}^2)}$ and we prove our claim.
	\qedhere
\end{proof}

Next, we provide the final proof.

\begin{proof}[Proof of Theorem \ref{thm:pacbayesrandom}]
	Note that sampling from the distribution ${Q(\vw,x)}$ as defined in eq.\eqref{eq:Qwx} is NP-hard in general, thus our plan is to upper-bound the expectation in eq.\eqref{eq:QwtoQwx} by using the maximum over random structured outputs and latent variables sampled independently from a proposal distribution ${R(\vw,x)}$ with support on ${\Y_x \times \H_x}$.
	
	Let ${T(\vw,x)}$ be a set of $\m$ i.i.d. random structured outputs and latent variables drawn from the proposal distribution ${R(\vw,x)}$, i.e., ${T(\vw,x) \sim R(\vw,x)^{\m}}$.
	Furthermore, let $\T(\vw)$ be the collection of the $n$ sets ${T(\vw,x)}$ for all ${(x,y) \in S}$, i.e. ${\T(\vw) \equiv \{T(\vw,x)\}_{(x,y) \in S}}$ and thus ${\T(w) \sim \{R(\vw,x)^{\m}\}_{(x,y) \in S}}$.
	For clarity of presentation, define:
	\begin{align*}
	v(x,y,\yp,\hp,\vw) \equiv d(y,\yp,\hp) {\rm\ } \Ind{m(x,y,\yp,\hp,\vw) \leq 1}
	\end{align*}
	For sets ${T(\vw,x)}$ of sufficient size $\m$, our goal is to upper-bound eq.\eqref{eq:QwtoQwx} in the following form for all parameters ${\vw \in \W}$:
	\begin{align*}
	\frac{1}{n} \sum_{(x,y) \in S}{ \E_{(\yp,\hp) \sim Q(\vw,x)}[v(x,y,\yp,\hp,\vw)] } \leq \frac{1}{n} \sum_{(x,y) \in S}{ \max_{(\yh,\hh) \in T(\vw,x)}{v(x,y,\yh,\hh,\vw)} } + \BigO{\nicefrac{\log^{2}{n}}{\sqrt{n}}}
	\end{align*}
	Note that the above expression would produce a tighter upper bound than the maximum loss over all possible structured outputs and latent variables since ${\max_{(\yh,\hh) \in T(\vw,x)}{v(x,y,\yh,\hh,\vw)} \leq \max_{(\yh,\hh) \in \Y_x \times \H_x}}{v(x,y,\yh,\hh,\vw)}$.
	For analysis purposes, we decompose the latter equation into two quantities:
	\begin{align}
	A(\vw,S) & \equiv \frac{1}{n} \sum_{(x,y) \in S}{\left( \E_{(\yp,\hp) \sim Q(\vw,x)}[v(x,y,\yp,\hp,\vw)] - \E_{T(\vw,x) \sim R(\vw,x)^{\m}}\left[\max_{(\yh,\hh) \in T(\vw,x)}{v(x,y,\yh,\hh,\vw)}\right] \right)} \label{eq:A} \\
	B(\vw,S,\T(\vw)) & \equiv \frac{1}{n} \sum_{(x,y) \in S}{\left( \E_{T(\vw,x) \sim R(\vw,x)^{\m}}\left[\max_{(\yh,\hh) \in T(\vw,x)}{v(x,y,\yh,\hh,\vw)}\right] - \max_{(\yh,\hh) \in T(\vw,x)}{v(x,y,\yh,\hh,\vw)} \right)} \label{eq:B}
	\end{align}
	Thus, we will show that ${A(\vw,S) \leq \sqrt{\nicefrac{1}{n}}}$ and ${B(\vw,S,\T(\vw)) \leq \O(\nicefrac{\log^{2}{n}}{\sqrt{n}})}$ for all parameters ${\vw \in \W}$, any training set $S$ and all collections $\T(\vw)$, and therefore ${A(\vw,S) + B(\vw,S,\T(\vw)) \leq \O(\nicefrac{\log^{2}{n}}{\sqrt{n}})}$.
	Note that while the value of ${A(\vw,S)}$ is deterministic, the value of ${B(\vw,S,\T(\vw))}$ is stochastic given that $\T(\vw)$ is a collection of sampled random structured outputs.
	
	Fix a specific $\vw \in \W$.
	If data is separable then $v(x,y,\yp,\hp,\vw) = 0$ for all ${(x,y) \in S}$ and ${(\yp,\hp) \in \Y_x \times \H_x}$.
	Thus, we have ${A(\vw,S) = B(\vw,S,\T(\vw)) = 0}$ and we complete our proof for the separable case.\footnote{
		The same result can be obtained for any subset of $S$ for which the ``separability'' condition holds.
		Therefore, our analysis with the ``non-separability'' condition can be seen as a worst case scenario.}
	In what follows, we focus on the non-separable case.

	\paragraph{Bounding the Deterministic Expectation ${A(\vw,S)}$.}
	
	Here, we show that in eq.\eqref{eq:A}, ${A(\vw,S) \leq \sqrt{\nicefrac{1}{n}}}$ for all parameters ${\vw \in \W}$ and any training set $S$, provided that we use a sufficient number $\m$ of random structured outputs sampled from the proposal distribution.
	
	By well-known identities, we can rewrite:
	
	\begin{align} \refstepcounter{equation}
	A(\vw,S) & = \frac{1}{n} \sum_{(x,y) \in S}{\int_0^1{\left( \P_{(\yp,\hp) \sim R(\vw,x)}[v(x,y,\yp,\hp,\vw) < z]^{\m} - \P_{(\yp,\hp) \sim Q(\vw,x)}[v(x,y,\yp,\hp,\vw) < z] \right)}dz} \tag{\theequation.a}\label{eq:expectedtoprob} \\
	& \leq \frac{1}{n} \sum_{(x,y) \in S}{ \P_{(\yp,\hp) \sim R(\vw,x)}[v(x,y,\yp,\hp,\vw) < 1]^{\m} } \nonumber \\
%	& = \frac{1}{n} \sum_{(x,y) \in S}{ \P_{(\yp,\hp) \sim R(\vw,x)}[d(y,\yp,\hp) < 1 \vee H(x,y,\yp) - m(x,y,\yp,\vw) < 0]^{\m} } \nonumber \\
	& = \frac{1}{n} \sum_{(x,y) \in S}{ \left(1 - \P_{(\yp,\hp) \sim R(\vw,x)}[d(y,\yp,\hp) = 1 \ {\rm and} \  m(x,y,\yp,\hp,\vw) \leq 1]\right)^{\m} } \nonumber \\
	& \leq \frac{1}{n} \sum_{(x,y) \in S}{ \left( 1 - \min{\left( \P_{(\yp,\hp) \sim R(\vw,x)}[d(y,\yp,\hp) = 1] {\rm\ ,\ } \P_{(\yp,\hp) \sim R(\vw,x)}[m(x,y,\yp,\hp,\vw) \leq 1] \right)} \right)^{\m} } \nonumber \\
	& = \frac{1}{n} \sum_{(x,y) \in S}{ \max{\left( 1 - \P_{(\yp,\hp) \sim R(\vw,x)}[d(y,\yp,\hp) = 1] {\rm\ ,\ } \P_{(\yp,\hp) \sim R(\vw,x)}[ m(x,y,\yp,\hp,\vw) > 1] \right)}^{\m} } \nonumber \\
	& \leq \max{\left( \beta {\rm\ ,\ } \lexp{\frac{-1}{128\gamma^2\NormII{\vw}^2}} \right)}^{\m} \tag{\theequation.b}\label{eq:betaphmbound} \\
	& = \sqrt{1/n} \tag{\theequation.c}\label{eq:Abound}
	\end{align}
	\noindent where the step in eq.\eqref{eq:expectedtoprob} holds since for two independent random variables ${g,h \in [0,1]}$, we have ${\E[g] = 1-\int_0^1{\P[g < z] dz}}$ and ${\P[\max{(g,h)} < z] = \P[g < z]\P[h < z]}$.
	Therefore, $\E[\max{(g,h)}] = 1-\int_0^1{\P[g < z]\P[h < z] dz}$.
	For the step in eq.\eqref{eq:betaphmbound}, we used Assumption \ref{asm:maxdistortion} for the first term in the $\max$.
	For the second term in the $\max$, let ${\Delta \equiv \Phi(x,y, \hs) - \Phi(x,\yp, \hp)}$ where $\hs = \argmax_{h \in \H_x} \Phi(x,y,h)\cdot \vw$, then  ${m(x,y,\yp,\hp,\vw) = \dotprod{\Delta}{\vw}}$.
	From $\NormII{\Phi(x,y,h)} \leq \gamma$,  we have that $\NormII{\Delta} \leq 2\gamma$.
	By Assumption \ref{asm:lownorm}, we have that $\NormII{\E[\Delta]} \leq 1/(2\sqrt{n}) \leq 1/(2 \norm{\vw}_2)$.
	By Cauchy-Schwarz inequality we have ${\dotprod{\E[\Delta]}{\vw} \leq \norm{\E[\Delta]}_2 \norm{\vw}_2 \leq \norm{\vw}_2/(2 \norm{\vw}_2) \leq 1/2}$.
	Since ${\dotprod{\E[\Delta]}{\vw} \leq 1/2}$ and $\NormII{\Delta} \leq 2\gamma$, we apply Lemma \ref{lem:phmbound} in the step in eq.\eqref{eq:betaphmbound}.
	%For the step in eq.\eqref{eq:Abound}, let ${\alpha \equiv \max{\left(\frac{1}{\log{(1/\beta)}}, 128 s^2 \norm{w}_2^2\right)}}$.
	For the step in eq.\eqref{eq:Abound}, let ${\lambda \equiv \max{\left(\frac{1}{\log{(1/\beta)}}, 128\gamma^2\NormII{\vw}^2\right)}}$.
	Note that ${\max{\left( \beta {\rm\ ,\ } \lexp{\frac{-1}{128\gamma^2\NormII{\vw}^2}} \right)} = \exp{-1/\lambda}}$.
	Furthermore, let ${\m = \frac{1}{2} \lambda \log{n}}$.
	Therefore, ${\max{\left( \beta {\rm\ ,\ } \lexp{-\frac{1}{128\gamma^2\NormII{\vw}^2}} \right)}^{\m} = (\exp{-1/\lambda})^{\frac{1}{2} \lambda \log{n}} = \exp{\frac{-1}{2} \log{n}} = \sqrt{1/n}}$.

	\paragraph{Bounding the Stochastic Quantity ${B(\vw,S,\T(\vw))}$.}
	
	Here, we show that in eq.\eqref{eq:B}, ${B(\vw,S,\T(\vw)) \leq \O(\nicefrac{\log^{2}{n}}{\sqrt{n}})}$ for all parameters ${\vw \in \W}$, any training set $S$ and all collections $\T(\vw)$.
	For clarity of presentation, define:
	\begin{align*}
	g(x,y,T,\vw) \equiv \max_{(\yh,\hh) \in T} v(x,y,\yh,\hh,\vw)
	\end{align*}
	Thus, we can rewrite:
	\begin{align*}
	B(\vw,S,\T(\vw)) = \frac{1}{n} \sum_{(x,y) \in S}{\left( \E_{T(\vw,x) \sim R(\vw,x)^{\m}}[g(x,y,T(\vw,x),\vw)] - g(x,y,T(\vw,x),\vw) \right)}
	\end{align*}
	Let ${r_x \equiv |\Y_x \times \H_x|}$ and thus ${\Y_x \times \H_x \equiv \{(y_1,h_1) \dots (y_{r_x}, h_{r_x}) \}}$.
	Let ${\pi(x) = (\pi_1 \dots \pi_{r_x}) }$ be a permutation of ${ \{ 1 \dots r_x \} }$ such that $\dotprod{\Phi(x,y_{\pi_1},h_{\pi_1})}{\vw} < \dots < \dotprod{\Phi(x,y_{\pi_{r_x} }, h_{\pi_{r_x}} ) }{\vw}$.
	Let $\Pi$ be the collection of the $n$ permutations ${\pi(x)}$ for all ${(x,y) \in S}$, i.e. ${\Pi = \{\pi(x)\}_{(x,y) \in S}}$.
	From Assumption \ref{asm:linearordering}, we have that ${R(\pi(x),x) \equiv R(\vw,x)}$.
	Similarly, we rewrite ${T(\pi(x),x) \equiv T(\vw,x)}$ and ${\T(\Pi) \equiv \T(\vw)}$.
	
	Furthermore, let ${\W_{\Pi,S}}$ be the set of all ${\vw \in \W}$ that induce $\Pi$ on the training set $S$.
	For the parameter space $\W$, collection $\Pi$ and training set $S$, define the function class ${\G_{\W,\Pi,S}}$ as follows:
	\begin{align*}
	\G_{\W,\Pi,S} \equiv \{g(x,y,T,\vw) \mid \vw \in \W_{\Pi,S} \ {\rm and} \  (x,y) \in S \}
	\end{align*}
	Note that since ${|\Y_x \times \H_x| \leq r}$ for all ${(x,y) \in S}$, then ${|\cup_{(x,y) \in S}{\Y_x \times \H_x}| \leq \sum_{(x,y) \in S}|\Y_x \times \H_x| \leq nr}$.
	Note that each ordering of the $nr$ structured outputs completely determines a collection $\Pi$ and thus the collection of proposal distributions ${R(\vw,x)}$ for each ${(x,y) \in S}$.
	Note that since ${|\cup_{(x,y) \in S}{\PS_x}| \leq \ell}$, we consider ${\Phi(x,y,h) \in \R^\ell}$.
%	Note that we consider ${\Phi(x,y,h) \in \R^k}$.
	Although we can consider ${\vw \in \R^\ell}$, the vector $\vw$ is sparse with at most $\s$ non-zero entries.
	Thus, we take into account all possible subsets of $\s$ features from $\ell$ possible features.
	From results in \cite{Bennett56,Bennett60,Cover67}, we can conclude that there are at most ${(nr)^{2(\s-1)}}$ linearly inducible orderings, for a fixed set of $\s$ features.
	Therefore, there are at most ${\binom{\ell}\s (nr)^{2(\s-1)} \leq \ell^\s (nr)^{2 \s}}$ collections $\Pi$.
	
	Fix ${\delta \in (0,1)}$.
	By Rademacher-based uniform convergence\footnote{
		Note that for the analysis of ${B(\vw,S,\T(\vw))}$, the training set $S$ is fixed and randomness stems from the collection ${\T(\vw)}$.
		Also, note that for applying McDiarmid's inequality, independence of each set ${T(\vw,x)}$ for all ${(x,y) \in S}$ is a sufficient condition, and identically distributed sets ${T(\vw,x)}$ are not necessary.
	} and by a union bound over all ${\ell^\s (nr)^{2 \s}}$ collections $\Pi$, with probability at least ${1-\delta/2}$ over the choice of $n$ sets of random structured outputs, simultaneously for all parameters ${\vw \in \W}$:
	\begin{align} \label{eq:Bbound}
	B(\vw,S,\T(\vw)) \leq 2 {\rm\ } \Rademacher_{\T(\Pi)}(\G_{\W,\Pi,S}) + 3 \sqrt{\frac{\s (\log{\ell} + 2\log{(nr)}) + \log{(4/\delta)}}{n}}
	\end{align}
	\noindent where ${\Rademacher_{\T(\Pi)}(\G_{\W,\Pi,S})}$ is the \emph{empirical} Rademacher complexity of the function class ${\G_{\W,\Pi,S}}$ with respect to the collection ${\T(\Pi)}$ of the $n$ sets ${T(\pi(x),x)}$ for all ${(x,y) \in S}$.
%	For clarity, define:
%	%
%	\begin{align*}
%	\Delta_p(x,y,h,\yp,\hp) \equiv \begin{cases}
%	c(p,x,y,h) - c(p,x,\yp,\hp) & \text{if } p \in \PS(x) \\
%	0 & \text{otherwise}
%	\end{cases}
%	\end{align*}
%	%
	Let $\sigma$ be an $n$-dimensional vector of independent Rademacher random variables indexed by ${(x,y) \in S}$, i.e., ${\P[\sigma_{(x,y)}=+1]=\P[\sigma_{(x,y)}=-1]=1/2}$.
	The empirical Rademacher complexity is defined as:
	\begin{align} \refstepcounter{equation}
	\Rademacher_{\T(\Pi)}(\G_{\W,\Pi,S}) & \equiv \E_\sigma\left[ \sup_{g \in \G_{\W,\Pi,S}}{\left( \frac{1}{n} \sum_{(x,y) \in S}{\sigma_{(x,y)} g(x,y,T(\pi(x),x),\vw)} \right)} \right] \nonumber \\
	& = \E_\sigma\left[ \sup_{\vw \in \W_{\Pi,S}}{\left( \frac{1}{n} \sum_{(x,y) \in S}{\sigma_{(x,y)} \max_{(\yh,\hh) \in T(\pi(x),x)}{d(y,\yh,\hh) {\rm\ } \Ind{1 - m(x,y,\yh,\hh,\vw) \geq 0}}} \right)} \right] \nonumber \\
	& = \E_\sigma\left[ \sup_{\vw \in \W_{\Pi,S}}{\left( \frac{1}{n} \sum_{(x,y) \in S}{\sigma_{(x,y)} \max_{(\yh,\hh) \in T(\pi(x),x)}{ \hspace{-0.2in} d(y,\yh,\hh) {\rm\ } \Ind{ 1\geq \max_{h \in \H_x} \dotprod{\Phi(x,y,h)}{\vw} - \dotprod{\Phi(x,\yh,\hh)}{\vw} }}} \right)} \right] \nonumber \\
	& = \E_\sigma\left[ \sup_{\vw \in \R^\ell \setminus \{0\}}{\left( \frac{1}{n} \sum_{i \in \{1 \dots n\}}{\sigma_i \max_{j \in \{1 \dots \m\}}{d_{ij} {\rm\ } \Ind{1 \geq \max_{h \in \{1\ldots |\H_x|\}} \dotprod{z'_{ih}}{\vw} - \dotprod{z_{ij}}{\vw} }}} \right)} \right] \tag{\theequation.a}\label{eq:maxiversonnormlinear} \\
	& \leq \sum_{j \in \{1 \dots \m\}}{ \E_\sigma\left[ \sup_{\vw \in \R^\ell \setminus \{0\}}{\left( \frac{1}{n} \sum_{i \in \{1 \dots n\}}{\sigma_i {\rm\ } d_{ij} {\rm\ } \Ind{1 \geq \max_{h \in \{1\ldots |\H_x|\}} \dotprod{z'_{ih}}{\vw} - \dotprod{z_{ij}}{\vw} }} \right)} \right] } \tag{\theequation.b}\label{eq:maxtosum} \\
	& \leq \sum_{j \in \{1 \dots \m\}}{ \E_\sigma\left[ \sup_{\vw \in \R^\ell \setminus \{0\}}{\left( \frac{1}{n} \sum_{i \in \{1 \dots n\}}{\sigma_i {\rm\ } \Ind{1 \geq  \max_{h \in \{1\ldots |\H_x|\}} \dotprod{z'_{ih}}{\vw} - \dotprod{z_{ij}}{\vw} }} \right)} \right] } \tag{\theequation.c}\label{eq:composition} \\
	& \leq \sum_{j \in \{1 \dots \m\}}{ \E_\sigma\left[ \sup_{\tilde{\vw} \in \R^{\ell(|\H|+1)+1} \setminus \{0\}}{\left( \frac{1}{n} \sum_{i \in \{1 \dots n\}}{\sigma_i {\rm\ } \Ind{\dotprod{z_{ij}^{\H}}{\tilde{\vw}} \geq 0}} \right)} \right] } \tag{\theequation.d}\label{eq:removenorm} \\
	& \leq 2 \m \sqrt{\frac{(2\s+1) \log{(\ell(nr+1)+1)} \log{(n+1)}}{n}} \tag{\theequation.e}\label{eq:rademacher}
	\end{align}
	\noindent where in the step in eq.\eqref{eq:maxiversonnormlinear}, the terms ${\sigma_i}$, ${d_{ij}}$, ${z'_{ih}}$, $z_{ij}$ correspond to ${\sigma_{(x,y)}}$, ${d(y,\yh,\hh)}$, $\Phi(x,y,h)$ and $\Phi(x,\yh,\hh)$ respectively.
	Thus, we assume that index $i$ corresponds to the training sample ${(x,y) \in S}$, and that index $j$ corresponds to the structured output and latent variable ${(\yh,\hh) \in T(\pi(x),x)}$.
	Note that since $\Phi(x,y,h) \in \R^\ell$, thus the step in eq.\eqref{eq:maxiversonnormlinear} considers ${\vw,z'_{ih}, z_{ij} \in \R^\ell \setminus \{0\}}$ without loss of generality.
	The step in eq.\eqref{eq:maxtosum} follows from the fact that for any two function classes $\G$ and $\H$, we have that ${\Rademacher(\{\max{(g,h)} \mid g \in \G \ {\rm and} \  h \in \H \}) \leq \Rademacher(\G) + \Rademacher(\H)}$.
	The step in eq.\eqref{eq:composition} follows from the composition lemma and the fact that ${d_{ij} \in [0,1]}$ for all $i$ and $j$.
	The step in eq.\eqref{eq:removenorm} considers a larger function class,  we consider ${\tilde{\vw},z_{ij}^\H \in \R^{\ell(|\H|+1)+1} \setminus \{0\}}$. More detailed, for a fixed $i,j$, and $\vw \in \R^\ell$, we can construct the vectors $z_{ij}^\H = (1,-z'_{i1},\ldots,-z'_{i|\H|},z_{ij})$ and $\tilde{\vw}^{(t)} = (1,\vw^{(1)},\ldots,\vw^{(|H|)},\vw)$, where $\vw^{(l)}=\vw$ if $l=t$, and $\vw^{(l)}=\vzero$ otherwise.
	The step in eq.\eqref{eq:rademacher} follows from the Massart lemma, the Sauer-Shelah lemma and the VC-dimension of sparse linear classifiers.
	That is, for any function class $\G$, we have that ${\Rademacher(\G) \leq \sqrt{\frac{2 VC(\G) \log{(n+1)}}{n}}}$ where ${VC(\G)}$ is the VC-dimension of $\G$.
	Finally, note that $|\Y_x \times \H_x| \leq r$ then $|\H_x| \leq r, \forall (x,y) \in S$, and $|\H| = |\cup_{(x,y) \in S} \H_x| \leq nr$. Also, since $\vw$ is $\s$-sparse, we have that $\tilde{\vw}$ is $(2\s+1)$-sparse. 
	Then, by Theorem~20 of \cite{Neylon06}, ${VC(\G) \leq 2 (2\s+1) \log{(\ell(|\H|+1)+1)}}$ for the class $\G$ of sparse linear classifiers on ${\R^{\ell(|\H|+1)+1}}$, with ${3 \leq 2\s+1 \leq \frac{9}{20} \sqrt{\ell(|\H|+1)+1}}$.
	
	By eq.\eqref{eq:pacbayes}, eq.\eqref{eq:QwtoQwx}, eq.\eqref{eq:Abound}, eq.\eqref{eq:Bbound} and eq.\eqref{eq:rademacher}, we prove our claim.
	\qedhere
\end{proof}

%%%%%%%%%%%%%%%%%%%%%%%%%%%%%%%%%%%%%%%%%%%%%%%%%%%%%%%%%%%%%%%%%%%%%%%%%%%%%%%%%%%%%%%%%%%%%%%%%%%%%%%%%%%%%%%%%%%%%%%%%%%%%%%%%%%%
%%%%%%%%%%%%%%%%%%%%%%%%%%%%%%%%%%%%%%%%%%%%%%%%%%%%%%%%%%%%%%%%%%%%%%%%%%%%%%%%%%%%%%%%%%%%%%%%%%%%%%%%%%%%%%%%%%%%%%%%%%%%%%%%%%%%

% CLAIMS

%%%%%%%%%%%%%%%%%%%%%%%%%%%%%%%%%%%%%%%%%%%%%%%%%%%%%%%%%%%%%%%%%%%%%%%%%%%%%%%%%%%%%%%%%%%%%%%%%%%%%%%%%%%%%%%%%%%%%%%%%%%%%%%%%%%%
%%%%%%%%%%%%%%%%%%%%%%%%%%%%%%%%%%%%%%%%%%%%%%%%%%%%%%%%%%%%%%%%%%%%%%%%%%%%%%%%%%%%%%%%%%%%%%%%%%%%%%%%%%%%%%%%%%%%%%%%%%%%%%%%%%%%

\subsection{Proof of Claim \ref{clm:changeofmeasure}}
\label{app:proof_change_of_measure}

\begin{proof}
 For all ${(x,y) \in S}$ and ${\vw \in \W}$, by definition of the total variation distance, we have for any event ${\A(x,y,\yp,\hp,\vw)}$:
\begin{align*}
\left|\P_{(\yp,\hp) \sim R(\vw,x)}[\A(x,y,\yp,\hp,\vw)] - \P_{(\yp,\hp) \sim \Rp(\vw,x)}[\A(x,y,\yp,\hp,\vw)]\right| & \leq TV(R(\vw,x) \| \Rp(\vw,x))
\end{align*}
Let the event ${\A(x,y,\yp,\hp,\vw): d(y,\yp,h') = 1 \ {\rm and} \  1 - m(x,y,\yp,\hp,\vw) \geq 0}$.
Since ${R(\vw,x)}$ fulfills Assumption \ref{asm:maxdistortion} with value ${\beta_1}$ and since ${TV(R(\vw,x) \| \Rp(\vw,x)) \leq \beta_2}$, we have that for all ${(x,y) \in S}$ and ${\vw \in \W}$:
\begin{align*}
\P_{(\yp,\hp) \sim \Rp(\vw,x)}[\A(x,y,\yp,\hp,\vw)] & \geq \P_{(\yp,\hp) \sim R(\vw,x)}[\A(x,y,\yp,\hp,\vw)] - TV(R(\vw,x) \| \Rp(\vw,x)) \\
 & \geq 1 - \beta_1 - \beta_2
\end{align*}
\noindent which proves our claim.
\qedhere
\end{proof}

\subsection{Proof of Claim \ref{clm:permutations}}
\label{app:proof_permutations}

\begin{proof}
Since $\Y_x$ is the set of all permutations of $v$ elements, then $|\Y_x| = v!$.
In addition, since $d(y,\yp,h) = \frac{1}{v} \sum_{i=1}^v \Ind{y_i \neq \yp_i}$ and since $R(x)$ is a uniform proposal distribution with support on $\Y_x \times \H_x$, we have:
\begin{align}\refstepcounter{equation}
\P_{(\yp,\hp) \sim R(x)} [d(y,\yp,\hp) = 1] &= \P_{\yp} [d(y,\yp) = 1] \nonumber \\
											&= \frac{F(v)}{v!} \tag{\theequation.a} \label{eq:prob} \\
											&\geq 1 - 2/3. \nonumber
\end{align}
For a fixed $y$, the function $F(v)$ in step eq.\eqref{eq:prob} represents the number of permutations $\yp \in \Y_x$ such that $d(y,\yp,h) = 1$.
Moreover, $F(v)$ can be computed through the following recursion: $F(v) = (v-1)! \times (1 + \sum_{i=1}^{v-2} \frac{F(i)}{i!} )$.
The probability is then $F(v)/v!$, it can be seen that this probability converges as $v \to \infty$ through the following: $\lim_{v\to \infty} \frac{F(v+1)}{(v+1)!} - \frac{F(v)}{v!} = 0$. 
The probability converges to $0.3679$ approximately, while achieving a minimum value of $1/3$ at $v=3$.
Hence $\beta = 2/3$.
\qedhere
\end{proof}

\subsection{Proof of Claim \ref{clm:sparsedata}}
\label{app:proof_sparsedata}

\begin{proof}
Let ${\Delta \equiv \Phi(x,y,\hs)-\Phi(x,\yp,\hp)}$.
Let $p \in \PS_x$ be a superindex denoting the partitions,
i.e., for all ${p \in \PS_x}$, let ${\Delta^p \equiv \Phi(x,y,\hs)-\Phi(x,\yp,\hp)}$ for some ${(\yp,\hp) \in \Upsilon_x^p}$.
By assumption, since ${(\yp,\hp) \in \Upsilon_x^p}$ then ${|\Delta^p_p| \leq b}$ and ${(\forall q \neq p) {\rm\ } \Delta^p_q = 0}$.
%Note that ${\norm{\Delta^p}_1 = \sum_{q \in \PS(x)}{|\Delta^p_q|} = |\Delta^p_p| = b}$.
%Thus ${|\Delta^p_p|/\norm{\Delta^p}_1 = 1}$ and ${(\forall q \neq p) {\rm\ } \Delta^p_q/\norm{\Delta^p}_1 = 0}$.
Therefore:
\begin{align*}
\norm{\E_{(\yp,\hp) \sim R(x)}\left[ \Delta \right]}_2 & = \sqrt{\sum_{q \in \PS_x}{ \E_{(\yp,\hp) \sim R(x)}\left[\Delta_q \right]^2 }} \\
& \leq \sqrt{\sum_{q \in \PS_x}{ \E_{(\yp,\hp) \sim R(x)}\left[|\Delta_q|\right]^2 }} \\
 & = \sqrt{\sum_{q \in \PS_x}{ \left( \sum_{p \in \PS_x}{ \P_{(\yp,\hp) \sim R(x)}[(\yp,\hp) \in \Upsilon_x^p] \; |\Delta^p_q| } \right)^2 }} \\
 & = \sqrt{\sum_{q \in \PS_x}{ \left( \P_{(\yp,\hp) \sim R(x)}[(\yp,\hp) \in \Upsilon_x^q] \; |\Delta^q_q| \right)^2 }} \\
 &\leq \sqrt{|\PS_x| \left(\frac{b}{|\PS_x|}\right)^2} \\
 & = b/\sqrt{|\PS_x|}
\end{align*}
\noindent where we used the fact that for a uniform proposal distribution ${R(x)}$, we have ${\P_{(\yp,\hp) \sim R(\vw,x)}[(\yp,\hp) \in \Upsilon_x^q] = 1/|\PS_x|}$.
Finally, since we assume that ${n \leq |\PS_x|/(4b^2)}$, we have ${b/\sqrt{|\PS_x|} \leq 1/(2 \sqrt{n})}$ and we prove our claim.
\qedhere
\end{proof}

\subsection{Proof of Claim \ref{clm:densedata}}
\label{app:proof_densedata}

\begin{proof}
Let ${\Delta \equiv \Phi(x,y,\hs)-\Phi(x,\yp,\hp)}$.
By assumption ${|\Delta_p| \leq b/|\PS_x|}$ for all ${p \in \PS_x}$.
%Note that ${\norm{\Delta}_1 = \sum_{p \in \PS(x)}{|\Delta_p|} = |\PS(x)| \; b}$.
%Thus ${|\Delta_p|/\norm{\Delta}_1 = 1/|\PS(x)|}$ for all ${p \in \PS(x)}$.
Therefore:
\begin{align*}
\norm{\E_{(\yp,\hp) \sim R(\vw,x)}\left[ \Delta \right]}_2 & = \sqrt{\sum_{p \in \PS_x}{ \E_{(\yp,\hp) \sim R(\vw,x)} \left[\Delta_p\right]^2 }} \\
 & \leq \sqrt{\sum_{p \in \PS_x}{ \E_{(\yp,\hp) \sim R(\vw,x)} \left[|\Delta_p|\right]^2 }} \\
 & \leq \sqrt{|\PS_x| \left(\frac{b}{|\PS_x|}\right)^2} \\
 & = b/\sqrt{|\PS_x|}
\end{align*}
Finally, since we assume that ${n \leq |\PS_x|/(4b^2)}$, we have ${b/\sqrt{|\PS_x|} \leq 1/(2 \sqrt{n})}$ and we prove our claim.
\qedhere
\end{proof}

\subsection{Proof of Claim \ref{clm:greedylocal}}
\label{app:proof_greedy_alg}

\begin{proof}
Algorithm \ref{alg:greedylocal} depends solely on the linear ordering induced by the parameter $\vw$ and the mapping ${\Phi(x,\cdot)}$.
That is, at any point in time, Algorithm \ref{alg:greedylocal} executes comparisons of the form ${\dotprod{\Phi(x,y,h)}{\vw} > \dotprod{\Phi(x,\yh,\hh)}{\vw}}$ for any two pair of structured outputs and latent variables $(y,h)$ and $(\yh,\hh)$.
\qedhere
\end{proof}

\section{Discussion, Further Examples and Details of Experiments}
\label{appendix_discussion}

\subsection{Discussion}
\label{sec:discussion}

In this section, we discuss in more detail the inference problem. We also briefly discuss the non-convexity of the formulation in eq.\eqref{eq:slack_random}.

\paragraph{Inference on Test Data.}
The upper bound in Theorem \ref{thm:pacbayesrandom} holds simultaneously for all parameters ${\vw \in \W}$.
Therefore, our result implies that after learning the optimal parameter ${\hat{\vw} \in \W}$ in eq.\eqref{eq:slack_random} from \emph{training} data, we can bound the decoder distortion when performing \emph{exact} inference on \emph{test} data.
More formally, Theorem \ref{thm:pacbayesrandom} can be additionally invoked for a \emph{test} set $\Sp$, also with probability at least ${1-\delta}$.
Thus, under the same setting as of Theorem \ref{thm:pacbayesrandom}, the Gibbs decoder distortion is upper-bounded with probability at least ${1-2\delta}$ over the choice of $S$ and $\Sp$.
In this paper, we focus on learning the parameter of structured prediction models. We leave the analysis of approximate  inference on test data for future work.

\paragraph{A Non-Convex Formulation.}
As mentioned in Section \ref{sec:preliminaries}, all formulations with latent variables (eq.\eqref{eq:cccp},eq.\eqref{eq:slack_all}, and eq.\eqref{eq:slack_random}) are non-convex objectives. The motivation to use the margin re-scaling approach in the work of \citet{Yu09} is that the non-convex objective leads to a difference of two convex functions, which allows the use of CCCP \cite{yuille2002concave}. In the case of models without latent variables, \citet{sarawagi2008accurate} propose a method to reduce the problem of slack re-scaling to a series of modified margin re-scaling problems.
However, there are two main caveats in their approach.
First, the optimization is only heuristic, that is, it is not guaranteed to solve the slack rescaling objective exactly.
Second, their method is specific to the cutting plane training algorithm and does not easily extend to stochastic algorithms.
\citet{choi2016fast} propose efficient methods for finding the most-violating-label in a slack re-scaling formulation, given an oracle that returns the most-violating-label in a (slightly modified) margin re-scaling formulation. 
However, in the case of latent models, it is still unclear if this sort of reductions are possible for the slack re-scaling approach because of the maximization in the margin with respect to the latent space. 

We also note that one way to make the objective in eq.\eqref{eq:slack_all} convex is to replace the maximization in the margin by the latent variable $\hh$. However, this not only results in a looser upper bound of the Gibbs decoder distortion but also under performs with respect to the methods mentioned in this paper.

\paragraph{Randomizing the Latent Space.}
We note that in the definition of the margin, there is a maximization over the latent space $\H$. In this paper, we sample structured outputs and latent variables from some proposal distribution and these samples are used in the outer maximization in eq.\eqref{eq:slack_random}. While sampling latent variables from some proposal distribution in the maximization of the margin might be computationally appealing, the main issue is that this will lead to a looser upper bound of the Gibbs decoder distortion.

\subsection{Further examples for Assumption \ref{asm:maxdistortion}}
\label{app:examples}

For completeness, we present the examples provided in \cite{honorio2016} since we make use of the suggested $\beta$ values in our synthetic experiments. Although their proofs are given without using latent variables, it is straightforward to extend their claims by marginalizing on $h$.

\paragraph{Any type of structured output for binary distortion functions.}
\label{clm:anystruct}
Let ${\Y_x \times \H_x}$ be an arbitrary countable set of feasible decodings of $x$, such that ${|\Y_x| \geq 2}$ for all ${(x,y) \in S}$.
Let ${d(y,\yp,h) = \Ind{y \neq \yp}}$.
The uniform proposal distribution ${R(\vw,x) = R(x)}$ with support on ${\Y_x \times \H_x}$ fulfills Assumption \ref{asm:maxdistortion} with ${\beta = 1/2}$.

\paragraph{Directed spanning trees for a distortion function that returns the number of different edges.}
\label{clm:trees}
Let ${\Y_x}$ be the set of directed spanning trees of $v$ nodes.
Let ${A(y)}$ be the adjacency matrix of ${y \in \Y_x}$.
Let ${d(y,\yp,h) = \frac{1}{2 (v-1)} \sum_{ij}{|A(y)_{ij} - A(\yp)_{ij}|}}$.
The uniform proposal distribution ${R(\vw,x) = R(x)}$ with support on ${\Y_x \times \H_x}$ fulfills Assumption \ref{asm:maxdistortion} with ${\beta = \frac{v-2}{v-1}}$.

\paragraph{Directed acyclic graphs for a distortion function that returns the number of different edges.}
\label{clm:dags}
Let ${\Y_x}$ be the set of directed acyclic graphs of $v$ nodes and $b$ parents per node, such that ${2 \leq b \leq v-2}$.
Let ${A(y)}$ be the adjacency matrix of ${y \in \Y_x}$.
Let ${d(y,\yp,h) = \frac{1}{b(2v-b-1)} \sum_{ij}{|A(y)_{ij} - A(\yp)_{ij}|}}$.
The uniform proposal distribution ${R(\vw,x) = R(x)}$ with support on ${\Y_x \times \H_x}$ fulfills Assumption \ref{asm:maxdistortion} with ${\beta = \frac{b^2+2b+2}{b^2+3b+2}}$.

\paragraph{Cardinality-constrained sets for a distortion function that returns the number of different elements.}
\label{clm:cardsets}
Let ${\Y_x}$ be the set of sets of $b$ elements chosen from $v$ possible elements, such that ${b \leq v/2}$.
Let ${d(y,\yp,h) = \frac{1}{2 b} (|y-\yp| + |\yp-y|)}$.
The uniform proposal distribution ${R(\vw,x) = R(x)}$ with support on ${\Y_x \times \H_x}$ fulfills Assumption \ref{asm:maxdistortion} with ${\beta = 1/2}$.

\subsection{Additional Details of Experiments}
\label{app:experiments}

\paragraph{Synthetic Experiments.}
In order to generate each training sample ${(x,y) \in S}$, we generated a random vector $x$ with independent Bernoulli entries, each with equal probability of being $1$ or $0$. 
The latent space consists of vectors of the same size of $x$ but with only one entry being $1$, intuitively, this bit ``corrects'' one of the entries in $x$. 
After generating $x$, we set ${(y,h) = f_{\vw^*}(x)}$.
That is, we solved eq.\eqref{eq:inferenceall} in order to produce the structured output $y$, and disregard $h$.

We replaced the discontinuous 0/1 loss ${\Ind{z \geq 0}}$ with the convex hinge loss ${\max{(0,1+z)}}$, as it is customary. 
Note however, that even by using the hinge loss, the objective functions in eq.\eqref{eq:cccp}, eq.\eqref{eq:slack_all} and in eq.\eqref{eq:slack_random} are still non-convex with respect to $\vw$. This is due to the maximization over the latent space in the definition of the margin.
We used ${\lambda = 1/n}$ as suggested by Theorems \ref{thrm:pacbayesall} and \ref{thm:pacbayesrandom}, and we performed $30$ iterations of the subgradient descent method with a decaying step size ${1/\sqrt{t}}$ for iteration $t$.
For sampling random structured outputs and latent variables in eq.\eqref{eq:slack_random}, we implemented Algorithm \ref{alg:greedylocal} for directed spanning trees, directed acyclic graphs and cardinality-constrained sets.
We performed the local changes in Algorithm \ref{alg:greedylocal} as follows.  Given a pair $(\yh,\hh)$, making a local change to $(\yh,\hh)$ consists on iterating through all pairs $(\yp,\hp)$ where $\yh$ and $\yp$ differ only in one edge/element, and where the single entries in $\hh$ and $\hp$ are contiguous.
Finally, we used ${\beta = 0.67}$ for directed spanning trees, ${\beta = 0.84}$ for directed acyclic graphs, and ${\beta = 0.5}$ for cardinality-constrained sets, as prescribed by the examples given in Section \ref{app:examples}.

\paragraph{Image Matching.}
Ground truth is provided in the Buffy Stickmen dataset for measuring performance on a test set.
The authors in \cite{gane2014learning,volkovs2012efficient} did not use latent variables, and considered the mapping $\Phi(x,y) = \frac{1}{18} \sum_{i=1}^{18} (\psi(I,i) - \psi(I',y_i))^2$, where $\psi(I,k) \in \R^{128}$ are the SIFT descriptors at scale 5 evaluated at keypoint $k$.
We properly centered the coordinates independently on each frame to avoid modeling translations in $h$.
We use the mapping $\Phi(x,y,h) = (\Phi(x,y), \frac{1}{18} \sum_{i=1}^{18} \NormII{c(I,i)\times h - c(I',y_i)}^2)$, where $c(I,k) \in \R^2$ are the coordinates of keypoint $k$.
Intuitively, we are adding one extra feature that summarizes the change in rotation and scaling of the keypoints, i.e., $\Phi(x,y,h) \in \R^{129}$.

The learning is performed using the random formulation as in eq.\eqref{eq:slack_random}, and using local changes as in Algorithm \ref{alg:greedylocal} for sampling from the proposal distribution. 
As in the synthetic experiments, we also replaced the discontinuous 0/1 loss ${\Ind{z \geq 0}}$ with the convex hinge loss ${\max{(0,1+z)}}$, and followed the local changes in Algorithm \ref{alg:greedylocal} for sampling from the proposal distribution.
The neighborhoods of the structures and latent variables were defined as follow: for a given permutation $y$, we considered $y'$ to be its neighbor, and vice versa, if they have only two mismatched entries.
Similarly, for a given $h$, we considered $h'$ to be its neighbor, and vice versa, if they have only one different entry.

\end{appendices}

\end{document}